\pdfoutput=1
\documentclass{article}
\usepackage{PRIMEarxiv}
\usepackage[utf8]{inputenc}
\usepackage{microtype}
\usepackage[dvipsnames,table]{xcolor}

\usepackage[T1]{fontenc}
\DeclareUnicodeCharacter{2113}{\ensuremath{\ell}}
\usepackage{amsmath,amssymb,amsthm}
\usepackage{graphicx,booktabs,array,longtable}
\usepackage{placeins}
\usepackage{float}
\usepackage{algorithm}
\usepackage{algpseudocode}
\usepackage[round]{natbib}

\usepackage{xcolor}
\usepackage{tikz}
\usepackage{tcolorbox}
\tcbuselibrary{breakable}
\usepackage[font=small,labelfont=bf,skip=3pt]{caption}
\PassOptionsToPackage{hyphens}{url}
\usepackage[hidelinks]{hyperref}
\usepackage{xurl}
\hypersetup{colorlinks, citecolor=blue, linkcolor=red, urlcolor=green,
            pdftitle={Executing Causal Structure Learning with Linear-Attention Transformers},
            pdfauthor={Amartya Roy, Sayar Karmakar}}
\newtheorem{proposition}{Proposition}[section]
\newtheorem{lemma}[proposition]{Lemma}
\newtheorem{theorem}[proposition]{Theorem}
\newtheorem{corollary}[proposition]{Corollary}
\theoremstyle{definition}

\theoremstyle{remark}
\newtheorem{remark}[proposition]{Remark}
\newcommand{\R}{\mathbb R}

\newcommand{\soft}{\operatorname{soft}}
\newcommand{\tr}{\operatorname{tr}}

\newcommand{\diag}{\operatorname{diag}}
\newcommand{\relu}{\operatorname{ReLU}}

\newcommand{\ind}{\mathbf 1}
\newcommand{\normmax}[1]{\lVert #1\rVert_{\max}}

\definecolor{attnblue}{RGB}{31,105,180}
\definecolor{fforange}{RGB}{205,105,30}
\definecolor{stategreen}{RGB}{35,125,75}
\newtcolorbox{keybox}{colback=white,colframe=blue!55!black,
  boxrule=0.8pt,sharp corners,
  boxsep=1pt,top=3pt,bottom=3pt,left=4pt,right=4pt,
  before skip=5pt,after skip=5pt}
\newcommand{\keytitle}[1]{\textbf{\color{blue!55!black}#1}\par
  \vspace{1pt}}
\newcommand{\numlabel}[1]{\refstepcounter{proposition}\label{#1}}
\newcommand{\papertitle}{Executing Causal Structure Learning with
  Linear-Attention Transformers}

\title{\papertitle}
\author{
  Amartya Roy \\ School of Interdisciplinary Research \\ Indian Institute of Technology Delhi \\ \texttt{srz248670@iitd.ac.in} \And
  Sayar Karmakar \\ University of Florida \\ \texttt{sayarkarmakar@ufl.edu}
}

\begin{document}
\maketitle
\thispagestyle{empty}

\begin{abstract}
Transformers can execute algorithms on data given in their input. We ask
whether they can do the same for causal discovery. We study a standard
continuous method that repeatedly updates a candidate causal graph while
enforcing acyclicity. We explicitly construct a fixed-weight transformer
whose forward pass exactly reproduces one update of this method, so
repeated blocks reproduce its optimization trajectory. The transformer
carries the current graph and the algorithm's multiplier between
updates. We show that retaining the multiplier is essential for exact
execution, since different multiplier values can lead to different next
updates. We also give conditions under which, within a fixed stage, the
number of updates needed to reach a target accuracy can be computed in
advance and rounding errors stay bounded as depth grows. Experiments
show that the constructed block agrees with a reference update to
floating-point precision, while arithmetic replay on synthetic data and
seven published benchmark network topologies inherits the reference
solver's successes and failures. This separates accurate algorithm
execution from accurate causal recovery. In contrast, the ordinary
attention models tested under our training budgets do not reliably
execute the update or transfer to larger graphs. Whether gradient
training can learn an executor in the architecture class of the
construction remains open.
\end{abstract}

\section{Introduction}
\label{sec:intro}

Can a transformer carry out causal discovery by executing a known
solver, with weights fixed before seeing the data? The question grows out
of a line of work on what transformers can compute from their prompt.
Constructive studies of in-context learning show how a transformer can
compute on examples it is shown \citep{Garg2022WhatCT} and even implement
gradient descent in its forward pass
\citep{Akyrek2022WhatLA,vonOswald2022TransformersLI,Ahn2023TransformersLT}.
Later constructions go further, building networks that carry out proximal
updates, looped programs and graph algorithms
\citep{Bai2023TransformersAS,Giannou2023LoopedTA}. What these results
share is that the weights are \emph{specified} rather than fitted, so the
forward pass is a known computation and one can say exactly what it
computes, how deep it must be and where it loses precision.

Causal discovery looks like a natural next test, and a harder one. A
weighted directed graph over $p$ variables is a causal graph only if it
contains no directed cycle, and this global requirement couples every
edge to every other. Continuous formulations turn the requirement into a
smooth constraint: NOTEARS replaces the combinatorial search over acyclic
graphs by an equality constraint on the weighted adjacency matrix
\citep{Zheng2018DAGsWN}, and polynomial variants have the same acyclic
zeros \citep{Yu2019DAGGNNDS,Wei2020DAGsWN}. The resulting solvers
alternate between a gradient-and-threshold step on the edge weights and an
update of a multiplier that tightens the acyclicity penalty. A
transformer that is to execute such a solver must do more than evaluate a
gradient. It must evaluate a matrix-power constraint, apply a sparse
proximal step, respect an allowed-edge mask, and keep a second piece of
state, the multiplier, available for the next update.

This is the problem we take on. We study the complete update of a
continuous DAG solver, which combines a least-squares gradient, a
polynomial acyclicity penalty, a sparse proximal step and a multiplier
update. Each ingredient is established; the question is how to execute
their composition in a transformer and retain the state the next update
needs. The state is the crux. The next update depends on both the weighted
adjacency matrix $W$ and the multiplier $\alpha$, so we encode this
complete state in the prompt and build a fixed-weight block that maps it to
the next state, leaving the penalty schedule and stopping rule external.
Figure~\ref{fig:dynamics} shows the resulting correspondence between the
solver and transformer trajectories.

\begin{figure}[H]
\centering
\includegraphics[width=0.48\textwidth,trim=0 100.8pt 0 0,clip]{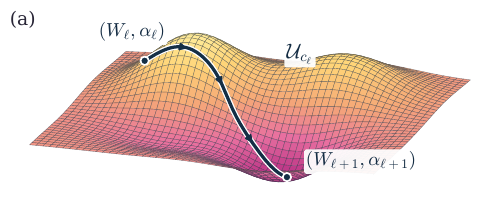}\hfill
\includegraphics[width=0.48\textwidth,trim=0 0 0 100.8pt,clip]{dynamics_correspondence.pdf}
\caption{\textbf{One update, run two ways.} The solver carries the
edge matrix $W_\ell$ and multiplier $\alpha_\ell$; the fixed-weight
transformer carries the same state in $Z_\ell$. With identical controls,
the two transitions agree exactly (Theorem~\ref{thm:realization}).
The surfaces are schematic.}
\label{fig:dynamics}
\end{figure}

Our answer is an explicit construction. A single set of weights, shared by
all graphs with up to $p_{\max}$ variables, executes one update exactly
with depth $O(\log p_{\max})$ and width $O(p_{\max})$, and we prove a
matching depth order for linear attention with bilinear/ReLU sublayers and
affine encoding and readout (Section~\ref{sec:realization}). Because the
multiplier is carried from one block to the next rather than recomputed,
it is the quantity most easily lost, and we show that losing it changes
the next update; we also bound the error this causes and explain why
constraint evaluations need care as the graph approaches acyclicity
(Section~\ref{sec:promptmatters}). Within a fixed penalty stage we give
conditions under which the number of updates needed for a target accuracy
can be computed in advance (Section~\ref{sec:stage}), and we relate the
solver's output to the true graph (Section~\ref{sec:recovery}).

The experiments then ask whether this works in practice.\footnote{Code:
\url{https://github.com/Amartya-Roy/Amartya-Roy-Executing-Causal-Structure-Learning-with-Linear-Attention-Transformers}}
The constructed block agrees with an independent reference implementation
of the update to floating-point precision, and repeated blocks reproduce
the solver's trajectory. Replaying the updates on synthetic data and on
seven published benchmark network topologies gives the same graphs as the
reference solver, including its mistakes: reproducing the solver's
updates also reproduces its statistical and optimization errors, so
accurate execution should not be confused with accurate causal recovery
(Section~\ref{sec:experiments}). Ordinary attention models trained under
our budgets, by contrast, do not reliably execute the update or transfer to
larger graphs, and whether gradient training can find an executor in the
architecture class of the construction remains open. We stress the scope
of the claim. The block executes the Poly-prox update for supplied
controls; it does not choose the schedule, run line searches or decide when
to stop, and the reference solver serves as a well-defined target rather
than a state-of-the-art estimator.

The rest of the paper follows this story.
Section~\ref{sec:related} reviews related work and
Section~\ref{sec:setup} states the solver and its reference update.
Section~\ref{sec:realization} builds the block, and
Sections~\ref{sec:promptmatters}--\ref{sec:recovery} study what is lost
through the multiplier, the constraint and finite precision, and how
exact execution relates to recovery.
Section~\ref{sec:experiments} reports the experiments and
Section~\ref{sec:discussion} concludes. Proofs and experimental details
are in the appendices.

\section{Related Work}
\label{sec:related}
\paragraph{Algorithm execution by transformers.}
Constructed transformers implement gradient and second-order updates
\citep{Akyrek2022WhatLA,vonOswald2022TransformersLI,Ahn2023TransformersLT,Chen2023TransformersLT},
proximal methods \citep{Bai2023TransformersAS} and optimal-transport
iterations \citep{daneshmand2026incontext}; looped models emulate
programs and graph algorithms \citep{Giannou2023LoopedTA,deLuca2024SimulationOG}.
We construct the complete Poly-prox update and determine the depth,
stored information and numerical accuracy it requires.
Learning latent causal structure \citep{Nichani2024HowTL} and amortized
graph prediction \citep{Ke2022LearningTI,Lorch2022AmortizedIF,Wang2026TransformerII}
address different targets from executing a specified numerical update.

\paragraph{Continuous DAG learning.}
NOTEARS uses a matrix-exponential constraint \citep{Zheng2018DAGsWN};
polynomial constraints have the same acyclic zeros
\citep{Yu2019DAGGNNDS,Wei2020DAGsWN}. TMPI exploits repeated squaring
\citep{Zhang2022TruncatedMP}; DAGMA, GOLEM and SDCD offer alternative
formulations \citep{Bello2022DAGMALD,Ng2020OnTR,Nazaret2023StableDC}.
Exact execution still inherits the solver's statistical and optimization
limitations, including those discussed in Appendix~\ref{app:solver_diagnostics}.

\section{The Causal-Discovery Algorithm}
\label{sec:setup}
Let $X\in\R^{m\times p}$ hold $m$ i.i.d.\ centered observations of a
mean-zero $x\in\R^p$ generated by a linear SEM
$x_j=\sum_iW^\dagger_{ij}x_i+\epsilon_j$ with acyclic $W^\dagger$ and
independent, mean-zero noise with finite positive variances. With equal noise variances
$W^\dagger$ uniquely minimizes the \emph{population} least-squares loss
over acyclic supports \citep{Peters2012IdentifiabilityOG,Loh2013HighdimensionalLO,Chen2018OnCD}; Section~\ref{sec:recovery} relates the
solver's finite-sample output to it.

For $\Sigma=X^\top X/m$, we study
\begin{equation}
\begin{aligned}
\min_{W\in\mathcal W_M}\quad & F(W)=\underbrace{\tfrac1{2m}\lVert
X-XW\rVert_F^2}_{\mathcal L(W)}+\lambda\lVert W\rVert_1\\
\text{s.t.}\quad & h(W)=0,
\end{aligned}
\label{eq:objective}
\end{equation}
where $\mathcal W_M=\{W:W=M\odot W\}$ restricts $W$ to a binary
zero-diagonal allowed-edge mask $M$ ($M_{ij}=\ind\{i\ne j\}$ with no domain
knowledge), $W_{ij}$ is the weight of edge $i\to j$, and
\begin{equation}
h(W)=\tr\Big[\Big(I+\tfrac{W\odot W}N\Big)^{N}\Big]-p,
\qquad N=2^s,\quad s=\lceil\log_2p_{\max}\rceil.
\label{eq:h_def}
\end{equation}
With $N\ge p$, $h(W)=0$ exactly on matrices whose support graph $G(W)$ is
acyclic \citep{Yu2019DAGGNNDS,Wei2020DAGsWN}; we fix one $N$ for every
$p\le p_{\max}$. We solve \eqref{eq:objective} by a NOTEARS-style
augmented Lagrangian
$\mathcal J_{\rho,\alpha}(W)=\mathcal L(W)+\alpha h(W)+\tfrac\rho2h(W)^2$
with a proximal-gradient primal step \citep{Daubechies2003AnIT,Beck2009AFI}:
\begin{align}
W_{\ell+1}&=M\odot\soft_{\gamma_\ell\lambda}\big(W_\ell-\gamma_\ell
\nabla\mathcal J_{\rho_\ell,\alpha_\ell}(W_\ell)\big),
\label{eq:w_update}\\
\alpha_{\ell+1}&=\alpha_\ell+b_\ell\rho_\ell h(W_{\ell+1}),
\label{eq:alpha_update}
\end{align}
Throughout, $\lambda\ge0$, $\rho_\ell\ge0$ and $\gamma_\ell>0$; the schedule controls are externally supplied. Here $\soft_\tau(t)=\operatorname{sign}(t)\max\{|t|-\tau,0\}$, the step
size is $\gamma_\ell$, and $b_\ell\in\{0,1\}$ allows several primal steps
per multiplier update. Algorithm~\ref{alg:polyprox} collects one
iteration of \eqref{eq:w_update}--\eqref{eq:alpha_update}, and we call
this reference implementation \emph{Poly-prox}. The solver itself is
standard; our contribution is a transformer with analytically fixed
weights that executes it (Section~\ref{sec:realization}).
We choose Poly-prox because its fixed step consists of matrix products,
scalar products and ReLU soft-thresholding, allowing a finite exact
construction. Line searches and alternative optimizers are not implemented by this fixed block.

Proposition~\ref{prop:gradient} (Appendix~\ref{app:proof_gradient})
gives, with $B=W\odot W$ and $Y=I+B/N$, $\nabla\mathcal L(W)=\Sigma(W-I)$
and $\nabla h(W)=2W\odot(Y^{N-1})^\top$, both built from the same matrix
powers that compute $h$ itself; the Lasso proximal map is entrywise
soft-thresholding restricted to the allowed edges,
$M\odot\soft_{\gamma\lambda}(W)$.

A well-known feature of this constraint, shared by the
matrix-exponential version \citep{Wei2020DAGsWN}, is that its gradient
vanishes on the entire feasible set: $[\nabla h(W)]_{ij}$ is a sum over
weighted directed walks that would close a cycle through $i\to j$, so it
is nonzero exactly on edges that lie on a cycle and $\nabla h(W)=0$ as
soon as $G(W)$ is acyclic (Lemma~\ref{lem:vanishing},
Appendix~\ref{app:proof_vanishing}).

At every feasible point, then, the usual regularity condition for the
equality constraint fails, and a feasible point is a KKT point exactly
when it also minimizes the convex masked Lasso problem
(Appendix~\ref{app:proof_vanishing}); Section~\ref{sec:promptmatters} turns
the same fact into a statement about what an executor must retain.

\section{How the Transformer Executes an Update}
\label{sec:realization}\label{sec:attention}
We now construct a block that exactly executes one iteration of
Algorithm~\ref{alg:polyprox}. The prompt and the block are designed
together: the prompt is a numerical encoding of the solver's current
state, one row token per variable plus a single \emph{control token}
holding the scalar controls and $\alpha$. A \emph{register} is a
designated group of coordinates in the residual stream; a
\emph{persistent} one carries its value from one iteration to the next
(here $W$ and $\alpha$), a \emph{scratch} one is reset after each
update.

The encoded state has width $13p_{\max}+18$. Set $q=9p_{\max}+18$
and let $v=(0_{9p_{\max}},1,0,0_{16})$ contain the variable-token flag
and zero scratch. The control row uses
$\omega_\ell=(0_{9p_{\max}},0,1,\lambda,\alpha_\ell,
\rho_\ell,\gamma_\ell,b_\ell,0_{11})$.
For variable rows $i\le p$, with all matrix rows padded to $p_{\max}$,
\begin{equation}
Z_\ell=\left[\begin{array}{ccccc}
\textcolor{attnblue}{e_i} & \textcolor{attnblue}{\Sigma_{i,:}} &
\textcolor{attnblue}{M_{i,:}} & \textcolor{stategreen}{(W_\ell)_{i,:}} &
\textcolor{stategreen}{v}\\[2pt]
0_{p_{\max}} & 0_{p_{\max}} & 0_{p_{\max}} & 0_{p_{\max}} &
\textcolor{stategreen}{\omega_\ell}
\end{array}\right].
\label{eq:prompt_encoding}
\end{equation}
The last block contains nine wide scratch registers and eighteen
scalar coordinates for flags, controls and scratch. Appendix~\ref{app:coordinates} indexes these coordinates;
Appendices~\ref{app:selectors}, \ref{app:schedule}, \ref{app:invariants}
and~\ref{app:cost} give the selector maps, schedule, invariants and cost.
We use residual, \emph{unnormalized} linear heads
$O(Z)=(Z\Theta_Q)(Z\Theta_K)^\top Z\Theta_V$, without softmax or layer
normalization, and feedforward sublayers
\begin{equation}
\mathrm{FF}(Z)=\relu(ZU_1)U_2+[(ZG_1)\odot(ZG_2)]G_3,
\label{eq:ff}
\end{equation}
the second term bilinear, as in gated feedforward units \citep{Shazeer2020GLUVI}. Soft-thresholding is exact with two ReLUs per
coordinate,
\begin{equation}
\soft_\tau(t)=\relu(t-\tau)-\relu(-t-\tau),\qquad\tau\ge0,
\label{eq:relu_soft}
\end{equation}
and the update inherits the mask from its inputs, so the thresholded
output is already supported on the allowed edges
(Appendix~\ref{app:construction}). Figure~\ref{fig:polyprox-transformer}
sets the solver and transformer updates side by side.

\begin{figure*}[t]
\centering
\footnotesize
\begin{minipage}[t]{0.485\textwidth}
\vspace{0pt}
\begin{algorithm}[H]
\caption{Poly-prox causal-discovery update}
\label{alg:polyprox}
\footnotesize
\begin{algorithmic}[1]
\Require Current state $(W_\ell,\alpha_\ell)$, covariance
$\Sigma=X^\top X/m$, mask $M$, and
$\lambda,\rho_\ell,\gamma_\ell,b_\ell$
\Ensure Updated state $(W_{\ell+1},\alpha_{\ell+1})$

\State $B_\ell\gets W_\ell\odot W_\ell,\;
Y_\ell\gets I+B_\ell/N$
\State $h_\ell\gets\operatorname{tr}(Y_\ell^N)-p$
\State $H_\ell\gets2W_\ell\odot(Y_\ell^{N-1})^\top$
\State $G_\ell\gets\Sigma(W_\ell-I)
+(\alpha_\ell+\rho_\ell h_\ell)H_\ell$
\State $W_{\ell+1}\gets
M\odot\operatorname{soft}_{\gamma_\ell\lambda}
(W_\ell-\gamma_\ell G_\ell)$
\State $\alpha_{\ell+1}\gets
\alpha_\ell+b_\ell\rho_\ell h(W_{\ell+1})$
\State \Return $(W_{\ell+1},\alpha_{\ell+1})$
\end{algorithmic}
\end{algorithm}
\end{minipage}
\hfill
\begin{minipage}[t]{0.485\textwidth}
\vspace{0pt}
\begin{algorithm}[H]
\caption{Fixed-weight transformer execution}
\label{alg:transformer-update}
\footnotesize
\begin{algorithmic}[1]

\Require Encoded state $Z_\ell$ from Eq.~\eqref{eq:prompt_encoding}
\Ensure Encoded state $Z_{\ell+1}$

\State \textcolor{stategreen}{
Read $(\Sigma,M,W_\ell,\alpha_\ell)$ and controls from $Z_\ell$.}

\State \textcolor{fforange}{
Form $B_\ell=W_\ell\odot W_\ell$ in the scratch registers.}

\State \textcolor{attnblue}{
Use the repeated-squaring stages to obtain
$h(W_\ell)$ and $\nabla h(W_\ell)$.}

\State \textcolor{attnblue}{
Assemble the augmented gradient $G_\ell$ from
$\Sigma(W_\ell-I)$ and the acyclicity term.}

\State \textcolor{fforange}{
Apply masked ReLU soft-thresholding to obtain $W_{\ell+1}$.}

\State \textcolor{stategreen}{
Compute $h(W_{\ell+1})$ and update the multiplier
$\alpha_{\ell+1}$.}

\State \textcolor{stategreen}{
Write $(W_{\ell+1},\alpha_{\ell+1})$ to the persistent registers,
reset scratch space, and \Return $Z_{\ell+1}$.}

\end{algorithmic}
\end{algorithm}
\end{minipage}

\caption{\textbf{One update, two implementations.} Left: one Poly-prox
iteration. Right: the same iteration as the fixed-weight transformer
carries it out on the encoded state $Z_\ell$ of
\eqref{eq:prompt_encoding}, line for line; blue is linear attention,
orange bilinear/ReLU, green persistent state.
Theorem~\ref{thm:realization} shows both return the same
$(W_{\ell+1},\alpha_{\ell+1})$.}
\label{fig:polyprox-transformer}
\end{figure*}

\numlabel{thm:realization}
\begin{keybox}
\keytitle{Theorem~\ref{thm:realization} (Exact execution of one Poly-prox update)}
\textit{Informal statement:} there is a block of $2\lceil\log_2
p_{\max}\rceil+8=O(\log p_{\max})$ linear-attention and bilinear/ReLU
layers, one weight set for every $p\le p_{\max}$, such that, for every
valid encoded state $Z_\ell$, the block executes
Algorithm~\ref{alg:transformer-update} exactly and therefore produces the
same $(W_{\ell+1},\alpha_{\ell+1})$ as one iteration of
Algorithm~\ref{alg:polyprox}, in exact real arithmetic, for every problem
size up to $p_{\max}$ at once.
\end{keybox}

Appendix~\ref{app:proofs} contains the proofs;
Appendix~\ref{app:proof_realization} states the construction theorem in full.
The block has width $13p_{\max}+18$, at most four heads per layer,
and nonzero weight magnitudes in $\{1/N,1,2\}$.

\paragraph{How the block evaluates the constraint.}
The block computes two matrix sequences together. Starting from
$S_0=B/N$ and $Q_0=0$, set
\begin{equation*}
S_{j+1}=2S_j+S_j^2,\qquad Q_{j+1}=Q_j+S_j+Q_jS_j.
\end{equation*}
After $s$ doublings, $S_s=Y^N-I$ and $Q_s=Y^{N-1}-I$ exactly
(Appendix~\ref{app:proof_realization}), so $h(W)=\tr S_s$ and,
since $\diag W=0$, $\nabla h(W)=2W\odot Q_s^\top$. The constraint and its gradient therefore share the same $s$ doubling
stages. Each doubling is one
attention layer: with one token per row, an unnormalized head whose
query reads $S$, whose key reads the identity register and whose value
reads $S$ writes $S^2$ row by row, and a second head writes $QS$ the
same way. The rest of the update --- the least-squares gradient
$\Sigma(W-I)$, the soft-threshold, and a second squaring pass for
$h(W_{\ell+1})$ --- uses a fixed number of layers
(Appendix~\ref{app:construction}).

\begin{corollary}[Exact execution over repeated updates]
\label{cor:exact_trajectory}
Let $Z_0$ encode $(W_0,\alpha_0)$ with the fixed data
$(\Sigma,M,\lambda)$, and at each iteration let an outside schedule write
only that iteration's controls $(\rho_\ell,\gamma_\ell,b_\ell)$ into the
control token. Stacking $L$ blocks then reproduces the first $L$
iterations of Algorithm~\ref{alg:polyprox} exactly in real arithmetic, at
depth $L(2s+8)$: each copy computes $\alpha_{\ell+1}$ and the next one
reads it, so the multiplier never leaves the network.
\end{corollary}

\paragraph{Why logarithmic depth is necessary.}
Any exact implementation using these operations and affine encoding and
readout also needs logarithmic depth. Here one layer consists of an
attention sublayer followed by a feedforward sublayer.

\numlabel{thm:depth_lower}
\begin{keybox}
\keytitle{Theorem~\ref{thm:depth_lower} (Minimum depth for exact execution)}
Fix $p_{\max}\ge2$ and $N=2^{\lceil\log_2p_{\max}\rceil}$. A $D$-layer
block of this type, with parameters fixed independently of its input,
any finite width and head count, and a token count fixed at each problem
size, that computes the Poly-prox primal update exactly for every valid
state satisfies
\begin{equation}
D\ge\left\lceil\log_6(4N-1)\right\rceil;
\label{eq:depth_primal_lower}
\end{equation}
if it computes the full state $(W^+,\alpha^+)$ exactly, then
\begin{equation}
D\ge\left\lceil\log_6\bigl(2N(4N-1)\bigr)\right\rceil.
\label{eq:depth_state_lower}
\end{equation}
Appendix~\ref{app:depth_lower} gives the proof.
\end{keybox}

Together with Theorem~\ref{thm:realization}, our results yield
$\Theta(\log p_{\max})$ as the least depth order for uniform exact
execution of the full fixed-$N$ update in this primitive class with
affine encoding/readout. This bound concerns the specified exponent-$N$
recurrence; it is not a lower bound for every DAG solver.
\section{Memory and Numerical Accuracy}
\label{sec:promptmatters}
The block retains $(W,\alpha)$ and recomputes the constraint in scratch
space. We now quantify the consequences of errors in this computation.
Write $z=(W,\alpha)$, $\lVert z\rVert_*=\max\{\lVert W\rVert_F,|\alpha|\}$,
and let $\Psi_\ell$ denote the exact update. A decoded transition is the
state returned by an implementation.

\numlabel{thm:tracking}
\begin{keybox}
\keytitle{Theorem~\ref{thm:tracking} (Accumulation of update errors)}
\textit{Informal statement:} if the decoded transition is
$\varepsilon_\ell$-close to the exact update $\Psi_\ell$ at every step,
and $\Psi_\ell$ is $\kappa_\ell$-Lipschitz, then, writing
$e_0=\lVert\widehat z_0-z_0\rVert_*$ for the initial state error,
\[
\lVert\widehat z_L-z_L\rVert_*\ \le\
e_0\prod_{u=0}^{L-1}\kappa_u+\sum_{t=0}^{L-1}
\varepsilon_t\prod_{u=t+1}^{L-1}\kappa_u.
\]
The proof is in Appendix~\ref{app:proof_tracking}.
\end{keybox}

\begin{remark}[Undamped multiplier error]
\label{rem:neutral}
By Lemma~\ref{lem:vanishing}, $\nabla h(W)=0$ whenever $W$ is acyclic, so
$\Psi_\ell(W,\alpha_1)-\Psi_\ell(W,\alpha_2)=(0,\alpha_1-\alpha_2)$ there
and $\kappa_\ell\ge1$. Taking $W_0=0$ and
$\lambda\ge\max_{i\ne j}|\Sigma_{ij}|$ keeps the exact trajectory at the
origin, while a decoder that adds $\varepsilon$ to $\widehat\alpha$ at
each step still meets the local-error bound yet ends at
$\lVert\widehat z_L-z_L\rVert_*=L\varepsilon$. What accumulates is
\emph{state} error, not graph error: near an acyclic state nothing pulls
$\alpha$ back, so an error made there is carried forward, and whether it
reaches $W$ depends on the next stage.
\end{remark}

Lemma~\ref{lem:distinguish} (Appendix~\ref{app:proof_distinguish}) shows
that if $[\nabla h(W)]_{ij}\ne0$ and, for one of two multiplier values
$\alpha_1\ne\alpha_2$, the pre-threshold entry $V_{ij}$ lies outside the
dead zone $[-\gamma\lambda,\gamma\lambda]$ of the soft-threshold, then the
augmented gradients differ and so do the updates; hence any
deterministic executor that returns $W^+$ exactly must retain
information distinguishing $\alpha_1$ from $\alpha_2$.

Near a stationary point, a small constraint gradient may need a large
multiplier to balance the loss gradient. The coefficient multiplying $\nabla h$ in the augmented gradient is the
\emph{effective multiplier} $c_k=\alpha_k+\rho_kh(W_k)$, which need
not have the same magnitude as $\alpha_k$ alone. All stationarity
residuals are taken \emph{inside} the linear subspace $\mathcal W_M$.
Write $d_M(W)$ for the stationarity residual of the convex \emph{masked
Lasso} problem, with acyclicity dropped, and call $W_k$
$\eta_k$-stationary when the augmented residual with coefficient $c_k$ is
at most $\eta_k$. Appendix~\ref{app:proof_multiplier} defines both
residuals and proves the following result.

\numlabel{prop:multiplier_growth}
\begin{keybox}
\keytitle{Proposition~\ref{prop:multiplier_growth} (Growth of the effective multiplier)}
\textit{Informal statement:} suppose a trajectory $W_k\to\bar W$ is
$\eta_k$-stationary for the augmented objective at every step with
$\eta_k\to0$, and $\bar W$ is feasible but \emph{not} a solution of the
masked Lasso problem. Then its effective multiplier
$c_k=\alpha_k+\rho_kh(W_k)$ must diverge in magnitude, at rate
\[
|c_k|\ \ge\ \frac{d_*}{4L_H\lVert W_k-\bar W\rVert_F}\ \longrightarrow\ \infty,
\]
where $d_*>0$ is the Lasso stationarity gap at $\bar W$ and $L_H$ is a
local Lipschitz constant for $\nabla h$ on a neighbourhood of $\bar W$.

\end{keybox}

Under these assumptions, $|c_k|$ grows as the constraint gradient
shrinks. It also multiplies any error in that gradient; error in $h$
affects $c_k$ through $\rho_k$. The next result shows when relative
accuracy keeps the resulting error bounded.

\begin{corollary}[Relative cycle accuracy at stage-stationary states]
\label{cor:relative_accuracy}
Let $C_k=\lVert M\odot\nabla\mathcal L(W_k)\rVert_F+\lambda\lVert
M\rVert_F+\eta_k$ as in the formal
Proposition~\ref{prop:multiplier_growth} of
Appendix~\ref{app:proof_multiplier}, and
\emph{assume} the covariance is exact and that the computed effective
multiplier and constraint gradient satisfy the relative bounds
$|\widetilde c_k-c_k|\le\nu|c_k|$ and
$\lVert\widetilde H-\nabla h\rVert_F\le\nu\lVert\nabla h\rVert_F$ for
some $\nu\in[0,1)$. If $W_k$ is $\eta_k$-stationary as above,
then the assembled gradient error obeys
$\lVert\widetilde g-g\rVert_F\le(2\nu+\nu^2)C_k$, so these \emph{assumed relative errors in $c_k$ and $\nabla h$}
stay controlled by $C_k$ even as $|c_k|\to\infty$.  The relative bounds are a hypothesis about the
implementation: Theorem~\ref{thm:realization} supplies $\nu=0$ in real
arithmetic, and floating point is measured in
Section~\ref{sec:experiments}.
\end{corollary}

Appendix~\ref{app:proof_relacc} proves the corollary and contrasts
relative and absolute perturbations. The two-node family in
Appendix~\ref{app:proof_multiplier} has $c_t\asymp t^{-3}$, showing
that the general growth bound need not be tight.
Figure~\ref{fig:exact} illustrates the relative-error bound and absolute-error
amplification (Appendix~\ref{app:exactplots}). Figure~\ref{fig:ex3_stationarity}
shows finite-trajectory stationarity diagnostics; these do not establish
the proposition's asymptotic premises (Appendix~\ref{app:ex3_stationarity}).

\paragraph{Bounding rounding errors.}
We follow rounding errors through every operation in the block,
including the second constraint evaluation and the replacement of
stored values.

\numlabel{thm:block_rounding}
\begin{keybox}
\keytitle{Theorem~\ref{thm:block_rounding} (Rounding error in the constructed block)}
Fix the selector weights and layer schedule of
Appendix~\ref{app:construction}, the matrix-product
kernels, accumulation order and casts. Under the scalar arithmetic model
\eqref{eq:rounding_model}, with exactly represented weights, flags and
mask and finite intermediates, propagating the nonnegative rules
\eqref{eq:enclosure_matmul}--\eqref{eq:enclosure_hadamard} through every
operation in \emph{both} squaring passes yields
\[
\lVert\widehat W^+-W^+\rVert_F\le\varepsilon_W,
\qquad |\widehat\alpha^+-\alpha^+|\le\varepsilon_\alpha,
\]
with the right sides in \eqref{eq:enclosure_output}. Uniform input
bounds give a uniform block bound, and for fixed input bounds both
enclosures tend to zero with the arithmetic and conversion errors
(Appendix~\ref{app:rounding}).
\end{keybox}

Appendix~\ref{app:rounding_model} states the arithmetic assumptions,
Appendix~\ref{app:rounding_enclosure} proves the enclosure, and
Appendix~\ref{app:rounding_writeback} bounds residual overwrites. The latter matter because
subtracting an old scratch value can lose relative accuracy even with
finite intermediates. These bounds are implementation-specific.

\paragraph{How accurately must the multiplier be stored?}
Lemma~\ref{lem:distinguish} shows why the multiplier matters. We now
bound the error from storing it imprecisely. Fix the data, $W$ and the
controls, and vary only $\alpha$; its encoding can use many coordinates.

\numlabel{prop:state_information}
\begin{keybox}
\keytitle{Proposition~\ref{prop:state_information} (Error from losing multiplier information)}
Let $H=\nabla h(W)$ and $I=[a,a+A]$, $A>0$. Suppose a nonempty set $J$ of
allowed entries stays on fixed active soft-threshold branches throughout
$I$, and $g_J=\gamma\lVert H_J\rVert_F>0$, with $H_J$ equal to $H$ on $J$
and zero elsewhere. If a deterministic encoder merges
$\alpha_1,\alpha_2\in I$, some decoded primal prediction errs by at least
$g_J|\alpha_1-\alpha_2|/2$. If the encoding takes at most $K\ge1$ values
on $I$ and the decoded primal error is uniformly at most
$\varepsilon$, then $\varepsilon\ge g_JA/(2K)$; a $B$-bit encoding needs
$\varepsilon\ge g_JA/2^{B+1}$. For an $L_D$-Lipschitz decoder,
$g_J|\alpha_1-\alpha_2|\le2\varepsilon+L_D\lVert
E(\alpha_1)-E(\alpha_2)\rVert$ (Appendix~\ref{app:state_information}).
\end{keybox}

Discarding $\alpha$ is the case $K=1$. The bound concerns information,
not a unique register layout. Appendix~\ref{app:state_information}
proves the result and gives a matching midpoint-quantization bound.

\section{How Many Updates Are Enough?}
\label{sec:stage}
We now ask how many updates suffice to reach a target accuracy while
the multiplier, penalty and step size stay fixed. The bound requires
contraction, a region the updates cannot leave, and a rounding-error
bound valid throughout that region. Fix finite $\alpha$, $\rho\ge0$,
$\gamma>0$, set $b=0$, and write
$\mathcal P(W)=M\odot\soft_{\gamma\lambda}
(W-\gamma\nabla\mathcal J_{\rho,\alpha}(W))$ for the primal map.
This describes one penalty stage, before the controller changes its settings.

\paragraph{When each update reduces the error.}
Let $\mathsf K_M(W)$ denote the Hessian of $\mathcal J_{\rho,\alpha}$
restricted to the allowed-edge subspace. With $H=\nabla h$ and
$c=\alpha+\rho h$, its quadratic form is
\begin{equation}
\begin{split}
\langle D,\mathsf K_M(W)D\rangle
={}&\tr(D^\top\Sigma D)
+c\,\langle D,\nabla^2h(W)[D]\rangle\\
&+\rho\langle H(W),D\rangle^2,
\qquad D\in\mathcal W_M.
\end{split}
\label{eq:stage_hessian_main}
\end{equation}
Its last term is positive semidefinite.
If $\mu I\preceq\mathsf K_M(W)\preceq L_J I$ uniformly on a convex
region, $\mu>0$, then for $0<\gamma<2/L_J$ the primal map contracts with
\begin{equation}
q=\max\{|1-\gamma\mu|,|1-\gamma L_J|\}<1.
\label{eq:spectral_q}
\end{equation}
The choice $\gamma=2/(L_J+\mu)$ gives $q=(L_J-\mu)/(L_J+\mu)$.

\paragraph{Checking the curvature bounds.}
For $\mathcal C=\{W\in\mathcal W_M:\lVert W-U\rVert_F\le r\}$,
put $R=\lVert U\rVert_F+r$. Lemma~\ref{lem:local_constants} gives
finite-$N$ constants $c_h(R),c_H(R),L_H(R),T_H(R)$ with
$c_h=O(R^4)$, $c_H=O(R^3)$, $L_H=O(R^2)$ and $T_H=O(R)$ as
$R\downarrow0$. A global sufficient lower bound is
\begin{equation}
\mu_{\mathrm g}=\lambda_M(\Sigma)
-(|\alpha|+\rho c_h(R))L_H(R),
\label{eq:mu_global_main}
\end{equation}
where $\lambda_M$ is the least eigenvalue of the loss Hessian on
nonempty allowed columns. Alternatively, compute the actual masked
Hessian at $U$, with extreme eigenvalues $m_U,L_U$, and set
\begin{equation}
\begin{split}
\delta_K(r)&=r\big(|\alpha+\rho h(U)|T_H(R)
+3\rho c_H(R)L_H(R)\big),\\
\mu_{\mathrm{loc}}&=m_U-\delta_K(r),\qquad
L_{\mathrm{loc}}=L_U+\delta_K(r).
\end{split}
\label{eq:local_certificate_main}
\end{equation}
If $\mu_{\mathrm{loc}}>0$, these are \emph{uniform} Hessian bounds on
$\mathcal C$ (Theorem~\ref{thm:local_hessian_certificate}). A two-node example in Appendix~\ref{app:local} passes this local test
even when the global bound fails. Proposition~\ref{prop:interval_stage_certificate}
gives a third test using interval arithmetic. Appendix~\ref{app:local}
proves all three.

\numlabel{thm:stage_certificate}
\begin{keybox}
\keytitle{Theorem~\ref{thm:stage_certificate} (Convergence within a fixed stage)}
\textit{Informal statement:} suppose $\mathcal P$ is $q$-Lipschitz,
$0\le q<1$, on $\mathcal C$, and every decoded transition stays in
$\mathcal W_M$ and differs from $\mathcal P$ by at most $\varepsilon$
\emph{uniformly on $\mathcal C$}. If
$d+\varepsilon\le(1-q)r$, where
$d=\lVert\mathcal P(U)-U\rVert_F$, then exact and decoded trajectories
starting in $\mathcal C$ stay there. The exact primal map has a unique
fixed point $W_\star$ in $\mathcal C$, and after $L$ \emph{updates},
\[
\lVert\widehat W_L-W_\star\rVert_F\le
q^LB_0+\frac{\varepsilon(1-q^L)}{1-q},
\]
where $B_0=\lVert\widehat W_0-U\rVert_F+d/(1-q)$.

\end{keybox}

For $0<q<1$, target $\mathrm{tol}>\varepsilon/(1-q)$, and $B_0>0$,
a sufficient number of updates is
\begin{equation}
L_{\mathrm{tol}}=\max\left\{0,
\left\lceil
\frac{\log\big(B_0/(\mathrm{tol}-\varepsilon/(1-q))\big)}
{\log(1/q)}\right\rceil\right\}.
\label{eq:integer_budget}
\end{equation}
 Each update is one constructed block of
$D_{\mathrm{blk}}=2\lceil\log_2p_{\max}\rceil+8$ layers. Thus the
execution depth is $D_{\mathrm{blk}}L_{\mathrm{tol}}$, with any initial
encoding cost separate.

Appendix~\ref{app:proof_stage} proves the invariant-ball and residual
bounds, including the $q=0$ and $B_0=0$ cases.
Corollary~\ref{cor:precision_stage} in Appendix~\ref{app:rounding_stage}
shows how to choose precision so the updates stay in the region and
reach the target accuracy.

\paragraph{When the bound applies in practice.}
The bound controls distance to a fixed point, which may still contain
cycles. None of the three tests passes on $270$ balls around $90$ saved
states. After further optimizing each fixed objective, the interval test
passes at $46/90$ centres with the recorded step size; $25$ saved stage
endpoints lie in these balls. The update bound excludes the earlier
iterations and this additional optimization. Appendix~\ref{app:practical_t3}
gives the coverage protocol (Table~\ref{tab:practical_t3_coverage})
and the subsequent tests with repeated transformer blocks.

\begin{figure*}[t]
\centering
\includegraphics[width=0.85\textwidth]{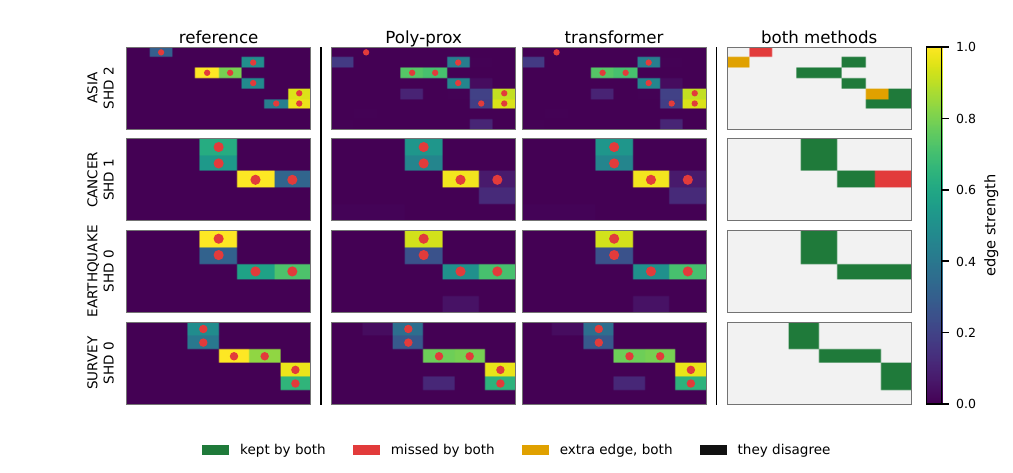}\\[4pt]
\includegraphics[width=0.85\textwidth]{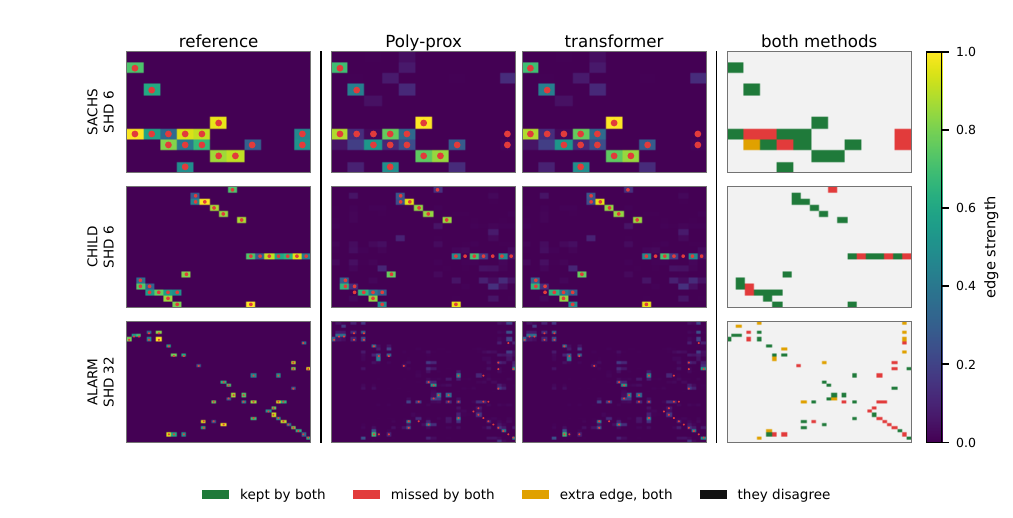}
\caption{\textbf{Benchmark topologies} (top: ASIA, CANCER, EARTHQUAKE,
SURVEY; bottom: SACHS, CHILD, ALARM). Median-SHD draw at $m=10^4$, SHD
beside each name; red dots are true edges; the column labelled ``transformer'' shows the direct arithmetic replay
of the block. Last
column, thresholded at $0.3$: green
kept by both, red missed by both, gold added by both; the black
disagreement category is empty on every row.}
\label{fig:inheritmaps}
\end{figure*}

\section{When Accurate Updates Recover the Graph}
\label{sec:recovery}
So far, we have measured agreement with Algorithm~\ref{alg:polyprox}.
Recovering the true graph also requires the solver itself to be accurate.
The next result separates error from computation, optimization and
finite data.

\numlabel{prop:recovery}
\begin{keybox}
\keytitle{Proposition~\ref{prop:recovery} (Graph recovery from an error bound)}
\textit{Informal statement:} if the total error
$\delta=\Delta_L+a_L+\zeta_m$ --- computation, optimization, and
statistical --- satisfies $\delta<\beta_{\min}/2$, where $\beta_{\min}$
is the weakest edge weight of a nonempty true graph, then thresholding the transformer's
output $\widehat W_L$ at any $\tau\in[\delta,\beta_{\min}-\delta)$ ---
keeping the entries with $\lvert\widehat W_{ij}\rvert>\tau$, so that the
endpoint $\tau=\delta$ is admissible --- recovers the true edge set
$G(W^\dagger)$ exactly.
\end{keybox}
The computation error $\Delta_L$ is what Theorem~\ref{thm:tracking}
controls; $a_L$ is optimization error against the chosen reference and
$\zeta_m$ is statistical error including regularization bias, which can
persist even with infinite data. Exact
execution sets the first term to zero, which is why
Section~\ref{sec:experiments} keeps the solver's own output
$W_{\mathrm{ref}}$ apart from the true graph $W^\dagger$.

Appendix~\ref{app:proof_recovery} defines the three errors and proves
the threshold claim, including its certified-stage specialization.

\section{Experiments}
\label{sec:experiments}
We ask how well the solver recovers graphs, whether the transformer
reproduces its updates, whether training learns those updates, and
whether replaying them yields the same graphs. In the solver comparison, Poly-prox, Poly-LBFGS and
NOTEARS use $\lambda=0.1$ and read-off threshold $0.3$; edge weights come
from $\pm[0.5,2]$. Appendix~\ref{app:experiments} gives the protocols,
metrics and failure counts.

\paragraph{How well does the reference solver recover graphs?}
On linear SEMs up to $20$ variables, continuous solvers and sort-regress
have lower structural Hamming distance (SHD) than PC and GES on raw data.
Poly-prox reaches median cycle constraint $7\cdot10^{-6}$, versus
$10^{-8}$ for L-BFGS, with slightly higher SHD but acyclic thresholded
outputs throughout the raw grid (Figure~\ref{fig:shd}). This establishes
a useful reference, not a leading estimator
(Appendix~\ref{app:solver_diagnostics}).

\paragraph{Does the transformer reproduce the solver update?}
We instantiate weights from $p_{\max}$ alone and compare literal
attention/FF execution with an independent matrix-power implementation.
Across $254$ finite-reference cases with $p_{\max}\in\{4,8,16,32\}$,
all sizes $p\le p_{\max}$, random masks and varied controls, the maximum
relative discrepancy is $3.0\cdot10^{-15}$ in $W$ and
$1.1\cdot10^{-13}$ in $\alpha$. One additional reference case overflows.
These tests encode each input anew. To test repeated execution, we also
pass the construction's full output directly through three blocks: $150$ comparisons on $50$ inputs have maximum Frobenius $W$
error $1.37\cdot10^{-16}$. Appendix~\ref{app:executor} distinguishes
these tests and documents a scratch-reset defect in the earlier
implementation. Agreement on one update does not ensure correct reuse
of the state.

The rounding bounds stay within the representable range in $129/135$
serial-arithmetic tests. In six float16 cases, the bounds exceed that
range even though the sampled outputs remain finite
(Appendix~\ref{app:new_precision}). Six float64 runs with fixed controls pass the full state between blocks,
stay inside the verified regions and reach the $10^{-6}$ distance target
within the predicted number of updates (Table~\ref{tab:practical_t3_rollouts}).
These runs test the bound within a stage; convergence of the outer
controller remains a separate question.

\begin{figure*}[t]
\centering
\includegraphics[width=0.90\textwidth]{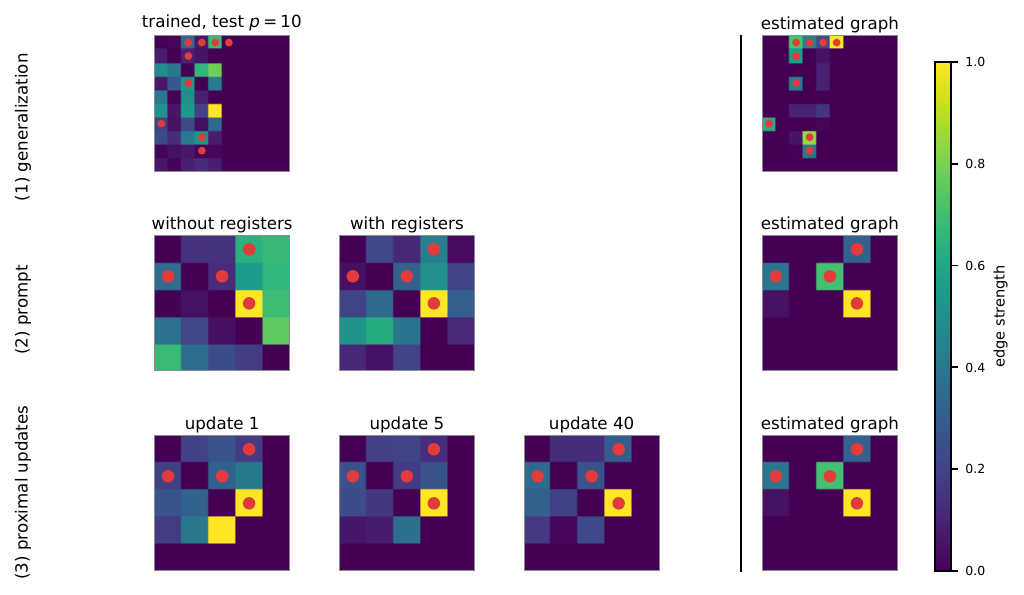}
\caption{\textbf{What the trained models produce.} Red dots mark true
edges. (1) A softmax model trained at $p=5$ and applied at $p=10$ after
$12$ updates. (2) At $p=5$, a collapsed-input linear model withholds
$W,\alpha$, while a softmax model receives them; this illustration
changes both architecture and state access. Table~\ref{tab:memory}
separately tests write-back with identical inputs and architectures.
(3) Reference-solver updates $1$, $5$ and $40$. The right column is
the solver's returned estimate. Colour shows $|W|$, scaled within each
panel; the rows use two illustrative problem instances.}
\label{fig:heatmaps}
\end{figure*}

\begin{figure*}[t]
\centering
\includegraphics[width=0.90\textwidth]{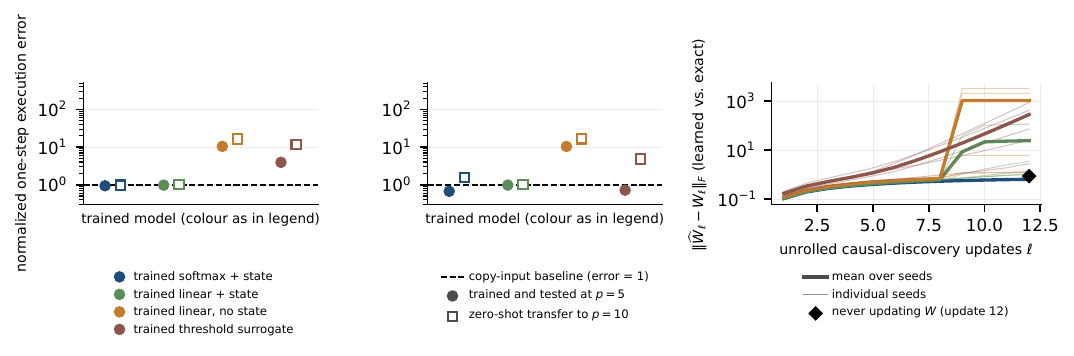}\\[4pt]
\includegraphics[width=0.62\textwidth]{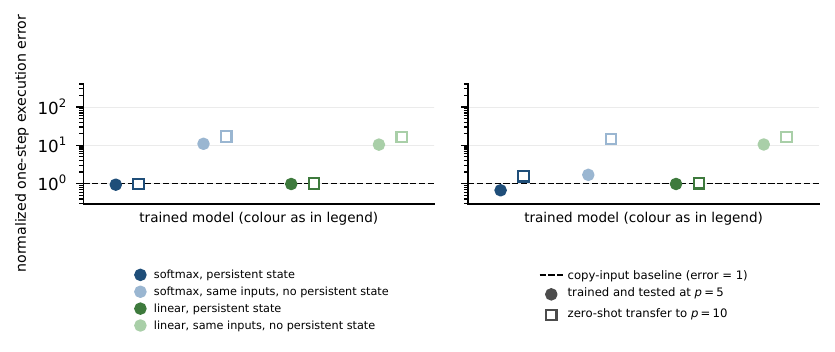}
\caption{\textbf{How close small trained transformers get to the exact update.}
One-step error is normalized so copying scores $1$ (dashed line).
Circles: trained and tested at $p=5$; squares: transfer to $p=10$.
\emph{Top:} smaller- and larger-corpus one-step errors, then smaller-corpus
rollout distances. \emph{Bottom:} identical state inputs with residual
write-back or prediction from scratch; the latter stays above copying.
Five training runs are attempted per condition; means and standard errors
use finite run summaries (failure counts in Tables~\ref{tab:train}
and~\ref{tab:memory}).}
\label{fig:execution}\label{fig:memory}
\end{figure*}

\paragraph{Can training learn the same update?}
We train width-$64$ attention blocks with four tied repeats at $p=5$,
using five seeds: $2460$ transitions over $12$ epochs or $24600$ over
$30$. The loss is mean squared error in $W^+$ plus $0.05$ times that in
$\alpha^+$. Appendix~\ref{app:learn} gives the data, optimizer and
treatment of failed runs.

One-step error is $\lVert\widehat W^+-W^+\rVert_F/
(\lVert W^+-W\rVert_F+10^{-8})$, so copying scores approximately $1$.
The $12$-step metric is an unnormalized Frobenius distance, for which
copying scores $0.86$. With the smaller corpus, state-input softmax and
linear attention score $0.94\pm0.00$ and $0.98\pm0.01$ on one step.
With the larger corpus, softmax reaches $0.67\pm0.01$ and a
learned-threshold map \citep{Gregor2010LearningFA} reaches $0.72\pm0.01$;
their trajectory errors improve to $0.40$ and $0.41$. Transfer to $p=10$
fails: one-step errors rise to $1.56\pm0.02$ and $4.8\pm0.4$,
respectively (Figure~\ref{fig:execution}, Table~\ref{tab:train}). Figure~\ref{fig:heatmaps}(1)
shows how a plausible edge pattern can still have inaccurate weights.

Holding state inputs fixed, predicting the next state from scratch
instead of adding a residual update raises the smaller-corpus errors
to $10.44\pm0.05$ (linear) and $10.98\pm0.38$ (softmax).
Figure~\ref{fig:memory} (bottom) and Table~\ref{tab:memory} isolate write-back;
Figure~\ref{fig:heatmaps}(2) separately illustrates a model whose state
inputs are withheld. These models have fewer layers and lack the construction's bilinear
sublayers. Whether training can learn the update using the full
architecture remains open.

\paragraph{Does the transformer recover the same graphs?}
Arithmetic replay uses the reference controls on ASIA, CANCER,
EARTHQUAKE, SURVEY, SACHS, CHILD and ALARM
\citep{Spiegelhalter1992Learning,Beinlich1989TheAM}. Ten Gaussian noise
draws per sample size yield $220$ datasets, with one fixed weighting per
topology (Appendix~\ref{app:benchmark_protocol}). Replay and Poly-prox
agree exactly in float64 on final weights, multipliers and thresholded
graphs across all datasets. Both recover EARTHQUAKE and SURVEY perfectly
and make identical mistakes on SACHS, CHILD and ALARM
(Figure~\ref{fig:inheritmaps}). Mean SHD on CHILD and ALARM is about
$6.3$ and $32$, versus $5$ and $21$ for NOTEARS. Replaying the updates also reproduces the solver's estimation errors.

\section{Discussion and Conclusion}
\label{sec:discussion}
A fixed-weight transformer executes the complete Poly-prox transition
in logarithmic depth, carrying the graph and multiplier between blocks.
The depth, information and rounding bounds specify its requirements;
experiments distinguish exact construction from partial learning by
ordinary attention models.

The guarantees apply to the stated update rule. The bound on the number
of updates requires contraction, a region the iterates cannot leave,
and a verified error bound. It excludes the adaptive controller and
does not by itself ensure acyclicity or graph recovery.
A learned executor should therefore be tested against the reference
transition on identical states and controls, then with its full output
fed back across updates. Graph accuracy alone cannot establish execution:
an exact executor also inherits the solver's mistakes. Learning within
the full constructed architecture remains open; initialization at the
constructed weights is one concrete direction.

\bigskip
\nocite{*}


\clearpage
\appendix
{\centering\Large\bfseries Supplementary Materials\par}
\vspace{1em}

\section{Proofs and Derivative Bounds}
\label{app:proofs}
This appendix proves the gradient, execution, state-error, conditioning,
update-count and graph-recovery results used in the main paper.

\subsection{Proof of Proposition~\ref{prop:gradient}}
\label{app:proof_gradient}
\numlabel{prop:gradient}
\noindent\textbf{Proposition~\ref{prop:gradient} (Update gradients).}
$\nabla\mathcal L(W)=\Sigma(W-I)$, and with $B=W\odot W$, $Y=I+B/N$,
\begin{equation}
\begin{aligned}
\nabla h(W)&=2W\odot(Y^{N-1})^\top,\\
\nabla\mathcal J_{\rho,\alpha}(W)&=\Sigma(W-I)+\big(\alpha+\rho h(W)\big)
  \nabla h(W).
\end{aligned}
\label{eq:grad_h}
\end{equation}
The proximal map of $\gamma\lambda\lVert W\rVert_1+\iota_{\mathcal
W_M}(W)$, where $\iota_{\mathcal W_M}$ is the convex indicator ($0$ on
$\mathcal W_M$ and $+\infty$ off it, not the $0/1$ indicator $\ind$ used
for the mask), is $M\odot\soft_{\gamma\lambda}(W)$.
\begin{proof}
$\nabla\mathcal L(W)=\Sigma(W-I)$ since
$\mathcal L(W)=\tfrac12\tr[(I-W)^\top\Sigma(I-W)]$ and $\Sigma$ is
symmetric. For $h$: $d\tr(Y^N)=N\tr(Y^{N-1}dY)$ by cyclicity of the trace,
and $dY=2(W\odot dW)/N$, so $dh=\langle2W\odot(Y^{N-1})^\top,dW\rangle$. The
chain rule gives the third identity; the proximal claim separates over
entries of $\mathcal W_M$.
\end{proof}

\subsection{Proof of Lemma~\ref{lem:vanishing}}
\label{app:proof_vanishing}
\numlabel{lem:vanishing}
\noindent\textbf{Lemma~\ref{lem:vanishing} (The constraint's gradient
vanishes exactly on DAGs).} For $i\ne j$,
\begin{equation}
[\nabla h(W)]_{ij}=2W_{ij}\sum_{k=1}^{N-1}\binom{N-1}kN^{-k}[B^k]_{ji}.
\label{eq:cycle_grad}
\end{equation}
Hence $[\nabla h(W)]_{ij}\ne0$ iff $W_{ij}\ne0$ and the edge $i\to j$ lies
on a directed cycle of $G(W)$ of length at most $N$. Consequently
$\nabla h(W)=0$ iff $h(W)=0$, i.e.\ iff $G(W)$ is acyclic.
\begin{proof}
Expand $Y^{N-1}=\sum_{k=0}^{N-1}\binom{N-1}kN^{-k}B^k$ in
\eqref{eq:grad_h}; the $k=0$ term is $I$, which vanishes off the diagonal,
giving \eqref{eq:cycle_grad}. Since $B=W\odot W\ge0$ entrywise, every
$B^k$ is entrywise non-negative and every coefficient
$\binom{N-1}kN^{-k}$ is positive, so the sum in \eqref{eq:cycle_grad}
vanishes exactly when every one of its terms does; walks of different
lengths cannot cancel. A nonzero $[B^k]_{ji}$ is a weighted
directed walk of length $k$ from $j$ to $i$ using edges of $G(W)$
(since $B_{ab}\ne0$ iff edge $a\to b$ is present); together with the edge
$i\to j$ this closes a cycle of length $k+1\le N$ through $i\to j$, and
conversely every such cycle contributes such a walk. Since every simple
cycle has length at most $p\le N$, this range of $k$ detects every
cycle in $G(W)$.
\end{proof}

\paragraph{Feasibility and KKT stationarity.}
At a feasible $W$, Lemma~\ref{lem:vanishing} removes the constraint
term from the finite-multiplier KKT equation. On $\mathcal W_M$ this
leaves $0\in M\odot\nabla\mathcal L(W)+\partial_M\psi_M(W)$, which
is equivalent to masked-Lasso optimality by convexity. For example,
take $p=N=2$, both off-diagonal entries allowed,
$\Sigma=\big(\begin{smallmatrix}1&r\\r&1\end{smallmatrix}\big)$
with $0\le\lambda<r<1$, and $W=0$. Then $h(W)=0$ and
$\nabla h(W)=0$, but both allowed least-squares gradient entries are
$-r$; no subgradient in $[-\lambda,\lambda]$ cancels them. This
feasible point is nonregular and is not a KKT point. The
stage-stationary family in Appendix~\ref{app:proof_multiplier}
approaches precisely such a non-Lasso-stationary feasible limit by
letting the effective multiplier diverge.

\subsection{Proof of Theorem~\ref{thm:realization}}
\label{app:proof_realization}
\noindent\textbf{Theorem~\ref{thm:realization} (Exact execution of one Poly-prox update).} Fix $p_{\max}\ge1$, $s=\lceil\log_2p_{\max}\rceil$,
$N=2^s$. There is a block of $2s+8$ layers, at most four linear heads
per layer of key dimension at most $p_{\max}$, embedding width
$13p_{\max}+18$, and at most $26p_{\max}+36$ ReLU units and
$\max\{p_{\max}+4,2p_{\max}+1\}$ bilinear units per feedforward sublayer, such that for
every $p\le p_{\max}$, every symmetric $\Sigma$, binary zero-diagonal
$M$, $W\in\mathcal W_M$, $\alpha\in\R$, $\rho,\gamma,\lambda\ge0$, and
$b\in\{0,1\}$, the block computes $W^+,\alpha^+$ exactly as in
\eqref{eq:w_update}--\eqref{eq:alpha_update}. The weights are the same
for every $p\le p_{\max}$ and every input, with nonzero magnitudes in
$\{1/N,1,2\}$; they depend on $p_{\max}$, not the data.
\begin{proof}
Put $Y=I+(W\odot W)/N$ and define
\begin{equation}
S_0=Y-I,\quad Q_0=0,\qquad
S_{j+1}=2S_j+S_j^2,\quad Q_{j+1}=Q_j+S_j+Q_jS_j.
\label{eq:squaring}
\end{equation}
Induction gives $I+S_j=Y^{2^j}$ and $I+Q_j=Y^{2^j-1}$: the first
recurrence squares $I+S_j$, and the second multiplies $I+Q_j$ by
$I+S_j$. Thus $h(W)=\tr S_s$ and
$\nabla h(W)=2W\odot Q_s^\top$, since $W$ has zero diagonal.
The three-head doubling layer and the eight preparation, threshold,
trace and write-back layers are specified coordinate by coordinate in
Appendix~\ref{app:construction}. They compute the least-squares gradient,
assemble the masked pre-threshold state, apply the exact ReLU identity,
then repeat the squaring pass on $W^+$ to update $\alpha$.
Two $s$-layer passes and eight other layers give depth $2s+8$.

The selector formulas in that appendix realize every matrix product,
transpose, sum and broadcast with at most four heads. Their feedforward
maps use signed ReLU pairs for linear writes and one bilinear unit per
coordinate product, giving the stated resource bounds. Forbidden primal
entries and padded columns remain zero, while the final layer clears
scratch exactly. Hence the same weights serve every $p\le p_{\max}$,
and the output encoding is a valid input to the next block. Induction
on stacked blocks proves Corollary~\ref{cor:exact_trajectory}.
\end{proof}

\subsection{Proof of Theorem~\ref{thm:tracking}}
\label{app:proof_tracking}
\noindent\textbf{Theorem~\ref{thm:tracking} (Accumulation of update errors).} If $\Psi_\ell$ is $\kappa_\ell$-Lipschitz in
$\lVert\cdot\rVert_*$ on a set $\mathcal K$ containing
$z_\ell,\widehat z_\ell$, and the decoded transition satisfies
$\lVert\widehat z_{\ell+1}-\Psi_\ell(\widehat
z_\ell)\rVert_*\le\varepsilon_\ell$, then with $e_0=\lVert\widehat z_0-
z_0\rVert_*$,
\begin{equation}
\lVert\widehat z_L-z_L\rVert_*\ \le\ e_0\prod_{u=0}^{L-1}\kappa_u+
\sum_{t=0}^{L-1}\varepsilon_t\prod_{u=t+1}^{L-1}\kappa_u.
\label{eq:track}
\end{equation}
\begin{proof}
Writing $e_\ell=\lVert\widehat z_\ell-z_\ell\rVert_*$, the triangle
inequality gives
$e_{\ell+1}\le\varepsilon_\ell+\kappa_\ell e_\ell$.
Unrolling this scalar recurrence proves \eqref{eq:track}, with empty
products equal to one and empty sums equal to zero.
\end{proof}

\subsection{Proof of Lemma~\ref{lem:distinguish}}
\label{app:proof_distinguish}
\begin{lemma}[Multiplier information must be retained]
\label{lem:distinguish}
Fix $\Sigma$, $M$, $W\in\mathcal W_M$, $\lambda$, $\rho$, $b$ and
$\gamma>0$, let $\alpha_1\ne\alpha_2$, and write
$V(\alpha)=W-\gamma\nabla\mathcal J_{\rho,\alpha}(W)$. If
$[\nabla h(W)]_{ij}\ne0$ for some $(i,j)$, and at least one of
$V_{ij}(\alpha_1)$, $V_{ij}(\alpha_2)$ lies outside the dead zone
$[-\gamma\lambda,\gamma\lambda]$ of the soft-threshold, then the
augmented gradients differ and so do the updates,
$W^+(\alpha_1)\ne W^+(\alpha_2)$. Hence any deterministic executor that
returns $W^+$ exactly must retain information distinguishing $\alpha_1$
from $\alpha_2$.
\end{lemma}
\begin{proof}
By \eqref{eq:grad_h} the least-squares term and $\rho h(W)$ do not
depend on $\alpha$, so
$\nabla\mathcal J_{\rho,\alpha_1}(W)-\nabla\mathcal
J_{\rho,\alpha_2}(W)=(\alpha_1-\alpha_2)\nabla h(W)$, which is nonzero
at $(i,j)$; hence
$V_{ij}(\alpha_1)-V_{ij}(\alpha_2)=-\gamma(\alpha_1-\alpha_2)
[\nabla h(W)]_{ij}\ne0$. By Lemma~\ref{lem:vanishing},
$[\nabla h(W)]_{ij}\ne0$ forces $W_{ij}\ne0$, so $M_{ij}=1$ and
$W^+_{ij}(\alpha)=\soft_\tau(V_{ij}(\alpha))$ with $\tau=\gamma\lambda$.
The map $\soft_\tau$ is nondecreasing, strictly increasing outside
$[-\tau,\tau]$, and zero on it; so for $a\ne b$,
$\soft_\tau(a)=\soft_\tau(b)$ only if both lie in $[-\tau,\tau]$, which
the hypothesis excludes. Thus $W^+_{ij}(\alpha_1)\ne W^+_{ij}(\alpha_2)$.
A deterministic executor given identical encodings for the two
multiplier values would return identical outputs, so its encodings of
$\alpha_1$ and $\alpha_2$ must differ.
\end{proof}

\subsection{Proof of Proposition~\ref{prop:multiplier_growth}}
\label{app:proof_multiplier}
\paragraph{Masked stationarity.}
The effective multiplier $c_k=\alpha_k+\rho_kh(W_k)$ need not have the
same magnitude as $\alpha_k$ alone. All stationarity residuals are taken
\emph{inside} the linear subspace $\mathcal W_M$. Let $\Pi_MA=M\odot A$
and $G_M(W)=\Pi_M\nabla\mathcal L(W)$. Define
$\psi_M:\mathcal W_M\to\R$ by $\psi_M(W)=\lambda\lVert W\rVert_1$, with
subspace subdifferential
\begin{equation}
\begin{aligned}
\partial_M\psi_M(W)&=\{\xi\in\mathcal W_M:\
\xi_{ij}=\lambda\operatorname{sign}(W_{ij})\ (W_{ij}\ne0),\\
&\qquad|\xi_{ij}|\le\lambda\ (M_{ij}=1,\ W_{ij}=0)\}.
\end{aligned}
\label{eq:masked_subgradient}
\end{equation}
Every $\xi$ is zero on forbidden entries and has norm at most
$\lambda\lVert M\rVert_F$. Set
\begin{equation}
d_M(W)=\min_{\xi\in\partial_M\psi_M(W)}\lVert G_M(W)+\xi\rVert_F,
\label{eq:masked_residual}
\end{equation}
the stationarity residual of the convex masked Lasso problem with
acyclicity dropped, so $d_M(W)=0$ exactly at its minimizers. For
$\eta_k\ge0$, $W_k$ is $\eta_k$-stationary if some
$\xi\in\partial_M\psi_M(W_k)$ satisfies
\begin{equation}
\lVert G_M(W_k)+c_k\nabla h(W_k)+\xi\rVert_F\le\eta_k.
\label{eq:masked_eta_residual}
\end{equation}
We use the subspace subgradient rather than the unbounded full-space
normal cone of an indicator of $\mathcal W_M$
(Appendix~\ref{app:proof_multiplier}).

Throughout this proof, $\psi_M$, $\partial_M\psi_M$ and $d_M$ are the
subspace quantities in \eqref{eq:masked_subgradient}--\eqref{eq:masked_residual}.
Their subgradients are bounded by $\lambda\lVert M\rVert_F$; no
full-space normal-cone term is included. Since $W\in\mathcal W_M$,
$\nabla h(W)$ is itself supported on $M$.
\noindent\textbf{Proposition~\ref{prop:multiplier_growth} (Growth of the effective multiplier).} Let $W_k\in\mathcal W_M$ and
$c_k=\alpha_k+\rho_kh(W_k)\in\R$, with no sign assumed, and suppose
$W_k$ is $\eta_k$-stationary for the augmented objective: there is
$\xi_k\in\partial_M\psi_M(W_k)$ with
\begin{equation}
\lVert M\odot\nabla\mathcal L(W_k)+c_k\nabla h(W_k)+\xi_k\rVert_F\le
\eta_k.
\label{eq:eta_stationary}
\end{equation}
Let $C_k=\lVert M\odot\nabla\mathcal L(W_k)\rVert_F+\lambda\lVert
M\rVert_F+\eta_k$. Then
\begin{equation}
d_M(W_k)-\eta_k\ \le\ |c_k|\,\lVert\nabla h(W_k)\rVert_F\ \le\ C_k.
\label{eq:sandwich}
\end{equation}
If $W_k\to\bar W$ with $h(\bar W)=0$, $\eta_k\to0$, and $\bar W$ is not a
masked-Lasso minimizer (so $d_*:=d_M(\bar W)>0$), then for any Lipschitz
bound $L_H$ of $\nabla h$ near $\bar W$,
\begin{equation}
|c_k|\ \ge\ \frac{d_*}{4L_H\lVert W_k-\bar W\rVert_F}\ \longrightarrow\
\infty
\label{eq:rate}
\end{equation}
for all large $k$; in particular $W_k\ne\bar W$ for large $k$.
\begin{proof}
Since $\xi_k\in\partial_M\psi_M(W_k)$, the residual $d_M(W_k)$ is at
most $\lVert M\odot\nabla\mathcal L(W_k)+\xi_k\rVert_F$. Inserting
$c_k\nabla h(W_k)$ and splitting it off, then using
\eqref{eq:eta_stationary} on what remains,
\begin{equation*}
d_M(W_k)\ \le\ \eta_k+|c_k|\,\lVert\nabla h(W_k)\rVert_F.
\end{equation*}
In the other direction,
$|c_k|\lVert\nabla h(W_k)\rVert_F=\lVert c_k\nabla h(W_k)\rVert_F\le
\eta_k+\lVert M\odot\nabla\mathcal L(W_k)\rVert_F+\lVert\xi_k\rVert_F
\le C_k$, because every element of $\partial_M\psi_M$ has Frobenius norm
at most $\lambda\lVert M\rVert_F$. This is \eqref{eq:sandwich}; only
$|c_k|$ enters, so the sign of $c_k$ plays no role.

The minimum in \eqref{eq:masked_residual} is attained on a nonempty
compact subgradient set. To see lower semicontinuity, take any
convergent sequence and a subsequence realizing the liminf of $d_M$;
choose a minimizing subgradient at each point. Their common norm bound
gives a further convergent subsequence, and the closed graph of
$\partial_M\psi_M$ places its limit in the subgradient set at the limit
point. Continuity of $G_M$ then shows that the liminf is at least the
limiting point's residual. Thus $d_*>0$ gives $d_M(W_k)\ge d_*/2$
for large $k$, and with $\eta_k\le d_*/4$ eventually, the left half of
\eqref{eq:sandwich} gives $|c_k|\lVert\nabla h(W_k)\rVert_F\ge d_*/4>0$.
In particular $\nabla h(W_k)\ne0$. Since $\nabla h(\bar W)=0$
(Lemma~\ref{lem:vanishing}) and $\nabla h$ is $L_H$-Lipschitz near $\bar
W$, $\lVert\nabla h(W_k)\rVert_F=\lVert\nabla h(W_k)-\nabla h(\bar
W)\rVert_F\le L_H\lVert W_k-\bar W\rVert_F$ for $W_k$ near $\bar W$, so
\begin{equation*}
|c_k|\ \ge\ \frac{d_*/4}{\lVert\nabla h(W_k)\rVert_F}\ \ge\
\frac{d_*}{4L_H\lVert W_k-\bar W\rVert_F},
\end{equation*}
which is \eqref{eq:rate}. The right side is finite only if
$W_k\ne\bar W$, which the previous step already guarantees, and it is
unbounded as $W_k\to\bar W$, so $|c_k|\to\infty$.
\end{proof}

\paragraph{A two-node example with faster multiplier growth.}
Take $p=N=2$, $M_{12}=M_{21}=1$,
$\Sigma=\big(\begin{smallmatrix}1&r\\r&1\end{smallmatrix}\big)$,
$0<\lambda<r<1$, and $W_t=\big(\begin{smallmatrix}0&t\\t&0
\end{smallmatrix}\big)$. Then $h(W_t)=t^4/2$, both allowed entries of
$\nabla h(W_t)$ equal $t^3$, and setting $c_t=(r-\lambda-t)/t^3$ makes
every allowed stationarity equation exact for small $t>0$. We see that $c_t\asymp t^{-3}$ as $W_t\to0$, faster than the general
$1/\mathrm{dist}$ rate above, so that rate is a lower bound rather than
the exact order everywhere. This example realizes the proposition's stationarity hypothesis exactly.

\subsection{Proof of Corollary~\ref{cor:relative_accuracy}}
\label{app:proof_relacc}
\noindent\textbf{Corollary~\ref{cor:relative_accuracy} (Relative cycle
accuracy).} Under the hypotheses stated there,
$\lVert\widetilde g-g\rVert_F\le(2\nu+\nu^2)C_k$.
\begin{proof}
Write $g=M\odot\nabla\mathcal L(W_k)+c_k\nabla h(W_k)$ for the assembled
gradient and $\widetilde g=M\odot\nabla\mathcal L(W_k)+\widetilde c_k
\widetilde H$ for the computed one; the covariance is exact by
hypothesis, so the least-squares term cancels and
\[
\widetilde g-g=(\widetilde c_k-c_k)\nabla h(W_k)
+\widetilde c_k\big(\widetilde H-\nabla h(W_k)\big).
\]
The first term has norm at most $\nu|c_k|\lVert\nabla h(W_k)\rVert_F$.
For the second, $|\widetilde c_k|\le(1+\nu)|c_k|$, so its norm is at
most $(1+\nu)|c_k|\cdot\nu\lVert\nabla h(W_k)\rVert_F$. Adding,
\[
\lVert\widetilde g-g\rVert_F\le\nu(2+\nu)\,|c_k|\,
\lVert\nabla h(W_k)\rVert_F.
\]
The right half of \eqref{eq:sandwich}, which uses only
$\eta_k$-stationarity \eqref{eq:eta_stationary} and no sign of $c_k$,
bounds $|c_k|\lVert\nabla h(W_k)\rVert_F$ by $C_k$, which is the claim.
\end{proof}
Relative accuracy of $h$ alone does not imply relative accuracy of
$c=\alpha+\rho h$, because the sum may cancel for signed $\alpha$.
For exact $c$ and absolute gradient error at most $\epsilon$, the bound
is instead $|c|\epsilon$, which permits amplification as $|c|$ grows.
Neither model forces a particular realized error. The literal block's
operation-level absolute enclosure is given in Appendix~\ref{app:rounding}.

\subsection{Proof of Theorem~\ref{thm:stage_certificate}}
\label{app:proof_stage}
\noindent\textbf{Theorem~\ref{thm:stage_certificate} (Convergence within a fixed stage).} Let $\mathcal C=\{W\in\mathcal
W_M:\lVert W-U\rVert_F\le r\}$, suppose $\mathcal P$ is $q$-Lipschitz on
$\mathcal C$ with $0\le q<1$, let $d=\lVert\mathcal P(U)-U\rVert_F$, and
suppose every approximate transition stays in $\mathcal W_M$ and is
within $\varepsilon$ of $\mathcal P$ on $\mathcal C$. If
$d+\varepsilon\le(1-q)r$, both exact and approximate iterations from
$\mathcal C$ stay in $\mathcal C$; $\mathcal P$ has a unique fixed point
$W_\star\in\mathcal C$; and, writing
$B_0=\lVert\widehat W_0-U\rVert_F+d/(1-q)$,
\begin{equation}
\lVert\widehat W_L-W_\star\rVert_F\le q^LB_0+
\frac{\varepsilon(1-q^L)}{1-q}.
\label{eq:stage_bound}
\end{equation}
For any $W\in\mathcal C$, $\lVert W-W_\star\rVert_F\le
\lVert W-\mathcal P(W)\rVert_F/(1-q)$. If only a decoded next state
$\widehat{\mathcal P}(W)$ is available, the bound is
$\big(\lVert W-\widehat{\mathcal P}(W)\rVert_F+\varepsilon\big)/(1-q)$:
the local error is added in the numerator.
\begin{proof}
For $W\in\mathcal C$,
$\lVert\widehat{\mathcal P}(W)-U\rVert_F\le\varepsilon+qr+d\le r$;
setting $\varepsilon=0$ proves exact invariance too. Banach's theorem
gives a unique fixed point of $\mathcal P$ in this closed ball, with
$\lVert W_\star-U\rVert_F\le q\lVert W_\star-U\rVert_F+d$, hence
$\lVert W_\star-U\rVert_F\le d/(1-q)$.
The one-step error satisfies
$e_{\ell+1}\le\varepsilon+qe_\ell$ for
$e_\ell=\lVert\widehat W_\ell-W_\star\rVert_F$.
Summing the geometric series and using $e_0\le B_0$ proves
\eqref{eq:stage_bound}. Finally,
$\lVert W-W_\star\rVert_F\le\lVert W-\mathcal P(W)\rVert_F
+q\lVert W-W_\star\rVert_F$ gives the residual certificate.
Replacing $\mathcal P(W)$ by its approximate value adds at most
$\varepsilon$ to the numerator.
\end{proof}

If $B_0=0$, zero updates suffice. If $q=0$, one update has error at most
$\varepsilon$, unless the initial bound already meets the target.
The fixed point may be cyclic; the update bound assumes fixed controls.

\subsection{Proof of Proposition~\ref{prop:recovery}}
\label{app:proof_recovery}
\noindent\textbf{Proposition~\ref{prop:recovery} (Graph recovery from an error bound).} Let $W^\circ$ be any reference matrix (e.g.\ an empirical
minimizer or the numerical limit) with $\normmax{W_L-W^\circ}\le a_L$
and $\normmax{W^\circ-W^\dagger}\le\zeta_m$. Under
Theorem~\ref{thm:tracking}, $\normmax{\widehat W_L-W^\dagger}\le
\Delta_L+a_L+\zeta_m=:\delta$, where $\Delta_L$ bounds
$\lVert\widehat W_L-W_L\rVert_F$. If the true graph is nonempty with
$\beta_{\min}=\min_{W^\dagger_{ij}\ne0}|W^\dagger_{ij}|$ and
$\delta<\beta_{\min}/2$, every threshold $\delta\le\tau<\beta_{\min}-
\delta$ recovers the exact support under the strict rule
$\widehat G=\{(i,j):\lvert\widehat W_{L,ij}\rvert>\tau\}$, and hence the
acyclic graph $G(W^\dagger)$. The rule must be strict at the left
endpoint: a true zero can attain $\lvert\widehat W_{L,ij}\rvert=\delta$,
so an inclusive test would keep it when $\tau=\delta$.
\begin{proof}
By the triangle inequality and
$\normmax{\widehat W_L-W_L}\le\lVert\widehat W_L-W_L\rVert_F\le\Delta_L$,
every true zero is estimated within $\delta\le\tau$ and every true edge
within at least $\beta_{\min}-\delta>\tau$.
\end{proof}

\begin{corollary}[Graph recovery within a verified stage]
\label{cor:certified_support}
Suppose Theorem~\ref{thm:stage_certificate} holds with error
$\varepsilon$, contraction $q<1$ and fixed point $W_\star$. Assume the
true graph is nonempty and there is a reference $W^\circ$ with
$\normmax{W_\star-W^\circ}\le a_\star$ and
$\normmax{W^\circ-W^\dagger}\le\zeta_m$. If
$\varepsilon/(1-q)+a_\star+\zeta_m<\beta_{\min}/2$, choose any
$\varepsilon/(1-q)<\mathrm{tol}<\beta_{\min}/2-a_\star-\zeta_m$ and use
the update bound \eqref{eq:integer_budget}, with the stated $q=0$ and $B_0=0$
conventions. Then thresholding $\widehat W_L$ at any
$\tau\in[\delta,\beta_{\min}-\delta)$, where
$\delta=\mathrm{tol}+a_\star+\zeta_m$, recovers the exact support under
the strict rule $|\widehat W_{L,ij}|>\tau$.
\end{corollary}

\begin{proof}
The stage bound gives $\lVert\widehat W_L-W_\star\rVert_F\le\mathrm{tol}$.
Add the two reference errors and apply Proposition~\ref{prop:recovery}.
For an empty true graph, any strict threshold $\tau\ge\delta$ removes
all entries. The reference-error bounds remain assumptions.
\end{proof}

\subsection{Derivative Bounds and Convergence Within a Stage}
\label{app:local}
All derivatives in this subsection act on $\mathcal W_M$, with its
Frobenius inner product, and $M$ is binary and zero-diagonal. Let
$\mathcal B_R^M=\{W\in\mathcal W_M:\lVert W\rVert_F\le R\}$.
Assume at least one edge is allowed, hence $p\ge2$ and $N\ge2$;
if none is allowed, the primal state space is $\{0\}$ and its dynamics
are trivial. For $R\ge0$, set $x=R^2/N$ and define
\begin{align}
c_h(R)&=(1+x)^N-1-R^2,\label{eq:sharp_ch}\\
a_N(R)&=(1+x)^{N-1}-1,\qquad
b_N(R)=\frac{N-1}{N}(1+x)^{N-2},\label{eq:sharp_ab}\\
d_N(R)&=
\begin{cases}
0,&N=2,\\
\dfrac{(N-1)(N-2)}{N^2}(1+x)^{N-3},&N\ge3,
\end{cases}\label{eq:sharp_dn}\\
c_H(R)&=2R a_N(R),\qquad
L_H(R)=2a_N(R)+4R^2 b_N(R),\label{eq:sharp_H}\\
T_H(R)&=12R b_N(R)+8R^3d_N(R).\label{eq:sharp_TH}
\end{align}
The powers are polynomial bounds, so at $R=0$ these formulas are
interpreted directly. The quantities $\nabla_M^2h$ and $D_M^2H$
below are the derivatives of $H=\nabla h$ along allowed directions,
with the output projected to $\mathcal W_M$ if needed.

\begin{lemma}[Zero-diagonal derivative bounds]
\label{lem:local_constants}
For $W,V\in\mathcal B_R^M$ and $D,E\in\mathcal W_M$,
\begin{align}
0\le h(W)&\le c_h(R),\qquad
\lVert H(W)\rVert_F\le c_H(R),\label{eq:sharp_value_bounds}\\
\lVert\nabla_M^2h(W)\rVert_{\mathrm{op}}&\le L_H(R),\qquad
\lVert D_M^2H(W)[D,E]\rVert_F
\le T_H(R)\lVert D\rVert_F\lVert E\rVert_F,\label{eq:sharp_derivative_bounds}\\
\lVert H(W)-H(V)\rVert_F&\le L_H(R)\lVert W-V\rVert_F,\label{eq:sharp_grad_lipschitz}\\
\lVert\nabla_M^2h(W)-\nabla_M^2h(V)\rVert_{\mathrm{op}}
&\le T_H(R)\lVert W-V\rVert_F.\label{eq:sharp_hessian_lipschitz}
\end{align}
Consequently, for fixed $\alpha\in\R$ and $\rho\ge0$, the masked
gradient of $\mathcal J_{\rho,\alpha}$ is Lipschitz on
$\mathcal B_R^M$ with
\begin{equation}
L_J(R)=\lVert\Sigma\rVert_2+
(|\alpha|+\rho c_h(R))L_H(R)+\rho c_H(R)^2.
\label{eq:sharp_LJ}
\end{equation}
The constants are no larger than the simpler, $N$-independent bounds
\begin{align*}
c_h(R)&\le e^{R^2}-1-R^2,&
c_H(R)&\le2R(e^{R^2}-1),\\
L_H(R)&\le2(e^{R^2}-1)+4R^2e^{R^2},&
T_H(R)&\le(12R+8R^3)e^{R^2}.
\end{align*}
In particular their orders at the origin are $R^4,R^3,R^2,R$,
respectively, uniformly in $N\ge2$ on bounded ranges of $R$.
\end{lemma}

\begin{proof}
Put $B=W\odot W$ and $A(B)=(I+B/N)^{N-1}-I$. Since
$\lVert B\rVert_2\le\lVert B\rVert_F\le R^2$ and $\tr B=0$,
\[
h(W)=\sum_{k=2}^N\binom NkN^{-k}\tr(B^k).
\]
All traces are nonnegative because $B$ is entrywise nonnegative.
For $k\ge2$, the Frobenius trace inequality and submultiplicativity give
$|\tr(B^k)|\le\lVert B\rVert_F\lVert B^{k-1}\rVert_F
\le\lVert B\rVert_F^2\lVert B\rVert_2^{k-2}\le R^{2k}$.
Summing gives \eqref{eq:sharp_ch}. The linear term disappears because the diagonal is zero; on
unrestricted matrices a bound of this order can fail.

Write $n=N-1$. Expansion and the induced matrix norm give
$\lVert A(B)\rVert_2\le a_N(R)$. Because $W\odot I=0$ on
$\mathcal W_M$, the exact gradient is
$H(W)=2W\odot A(B)^\top$, and hence
$\lVert H(W)\rVert_F\le2R a_N(R)$, using
$\lVert C\odot F\rVert_F\le\lVert C\rVert_F\lVert F\rVert_2$.
The identity cancellation is valid along every allowed direction.

For arbitrary matrix directions $F,G$, differentiate the finite power
without assuming commutativity. There are $n$ placements of its first
derivative and $n(n-1)$ ordered placements of its second derivative.
With $Y=I+B/N$ and $\lVert Y\rVert_2\le1+x$, this gives
\begin{align*}
\lVert DA(B)[F]\rVert_F&\le b_N(R)\lVert F\rVert_F,\\
\lVert D^2A(B)[F,G]\rVert_F&\le
d_N(R)\lVert F\rVert_F\lVert G\rVert_F.
\end{align*}
For $N=2$ the second derivative is zero. The directions of $B(W)$ are
$DB(W)[D]=2W\odot D$ and $D^2B(W)[D,E]=2D\odot E$, so their norms
are bounded by $2R\lVert D\rVert_F$ and
$2\lVert D\rVert_F\lVert E\rVert_F$, respectively. Thus
\[
DH(W)[D]=2D\odot A^\top
+2W\odot DA(B)[DB(W)[D]]^\top
\]
has norm at most $L_H(R)\lVert D\rVert_F$. Differentiating again yields
\begin{align*}
D^2H(W)[D,E]={}&2D\odot DA(B)[DB(W)[E]]^\top\\
&+2E\odot DA(B)[DB(W)[D]]^\top\\
&+2W\odot\big(D^2A(B)[DB(W)[D],DB(W)[E]]
+DA(B)[D^2B(W)[D,E]]\big)^\top.
\end{align*}
The first two terms contribute $8R b_N(R)$ to the bilinear operator
bound, and the last contributes $8R^3d_N(R)+4R b_N(R)$.
This proves the bound $T_H(R)$. Projection to allowed coordinates is
nonexpansive. Integration along the line segment between $W$ and $V$
in the convex ball proves the two Lipschitz inequalities.

For completeness, the full smooth Hessian, restricted after taking
directions in $\mathcal W_M$, is
\begin{equation}
\mathsf K_M(W)=\mathsf S_M+
(\alpha+\rho h(W))\nabla_M^2h(W)+\rho H(W)\otimes H(W),
\label{eq:full_masked_hessian}
\end{equation}
where $\mathsf S_M D=\Pi_M(\Sigma D)$ and
$(H\otimes H)D=\langle H,D\rangle H$. Its norm is bounded by
\eqref{eq:sharp_LJ}, proving the masked-gradient Lipschitz claim.
Finally, the binomial bounds $(1+R^2/N)^N\le e^{R^2}$ and
$\binom NkN^{-k}\le1/k!$ give the exponential bounds, and their
Taylor expansions give the stated orders. No finite-precision
assumption is used in these derivative bounds.
\end{proof}

\begin{corollary}[Stage contraction and global sufficient curvature]
\label{cor:stage_contraction}
Fix $\alpha\in\R$, $\rho\ge0$, $\lambda\ge0$ and $b=0$. If a convex
$\mathcal C\subseteq\mathcal W_M$ has masked strong-monotonicity
constant $\mu>0$ and masked-gradient Lipschitz constant $L_J$, then
$0<\gamma<2\mu/L_J^2$ gives contraction with
$q=(1-2\gamma\mu+\gamma^2L_J^2)^{1/2}<1$.
If, more specifically,
$\mu I\preceq\mathsf K_M(W)\preceq L_J I$ on $\mathcal C$, then
every $0<\gamma<2/L_J$ gives the sharper spectral factor
$q=\max\{|1-\gamma\mu|,|1-\gamma L_J|\}<1$.
For $\mathcal C\subseteq\mathcal B_R^M$, sufficient Hessian bounds are
$L_J=L_J(R)$ and
\begin{equation}
\mu_{\mathrm g}=\lambda_M(\Sigma)
-(|\alpha|+\rho c_h(R))L_H(R),\qquad
\lambda_M(\Sigma)=\min_{j:\pi_j\ne\varnothing}
\lambda_{\min}(\Sigma_{\pi_j\pi_j}),
\quad\pi_j=\{i:M_{ij}=1\},
\label{eq:sharp_mu}
\end{equation}
whenever $\mu_{\mathrm g}>0$. Empty allowed columns are omitted.
\end{corollary}

\begin{proof}
The allowed-edge proximal map is nonexpansive in the Frobenius norm.
For $D=W-V$ and $\Delta G=\Pi_M(\nabla\mathcal J(W)-
\nabla\mathcal J(V))$, strong monotonicity and Lipschitzness imply
\[
\lVert\mathcal P(W)-\mathcal P(V)\rVert_F^2
\le\lVert D-\gamma\Delta G\rVert_F^2
\le(1-2\gamma\mu+\gamma^2L_J^2)\lVert D\rVert_F^2.
\]
Under uniform Hessian bounds, instead write
$\Delta G=\overline{\mathsf K}D$, with
$\overline{\mathsf K}=\int_0^1\mathsf K_M(V+tD)\,dt$.
This symmetric operator has spectrum in $[\mu,L_J]$, so
$\lVert I-\gamma\overline{\mathsf K}\rVert_{\mathrm{op}}
\le\max\{|1-\gamma\mu|,|1-\gamma L_J|\}$.
The argument holds on every branch of the soft-threshold map.

For \eqref{eq:sharp_mu}, the loss quadratic form obeys
$\tr(D^\top\Sigma D)\ge\lambda_M(\Sigma)\lVert D\rVert_F^2$.
In \eqref{eq:full_masked_hessian},
$\rho\langle H,D\rangle^2\ge0$ is explicitly dropped only from
the lower bound, while
$|(\alpha+\rho h)\langle D,\nabla_M^2h[D]\rangle|
\le(|\alpha|+\rho c_h)L_H\lVert D\rVert_F^2$.
The upper bound retains the rank-one term through $\rho c_H^2$.
\end{proof}

For this \emph{global sufficient test}, positivity requires positive
definiteness of each nonempty $\Sigma_{\pi_j\pi_j}$, and therefore
$m\ge\max_j|\pi_j|$. The local stage Hessian also carries penalty curvature, so this
requirement belongs to the global test alone.
For fixed finite $\alpha,\rho$, $\mu_{\mathrm g}$ tends to
$\lambda_M(\Sigma)$ as $R\downarrow0$. The unsharpened bound $2e^{R^2}(1+2R^2)$ imposes the radius-independent
restriction $|\alpha|<\lambda_M(\Sigma)/2$; the sharpened $L_H(R)$
removes it.
A suitable radius around a given nonzero stage centre still has to be
checked.

\begin{theorem}[Uniform local Hessian certificate]
\label{thm:local_hessian_certificate}
Let $U\in\mathcal W_M$, $r>0$,
$\mathcal C=\{W\in\mathcal W_M:\lVert W-U\rVert_F\le r\}$ and
$R=\lVert U\rVert_F+r$. Let $m_U$ and $L_U$ be the least and greatest
eigenvalues of the \emph{actual} masked Hessian $\mathsf K_M(U)$,
including its rank-one penalty term. Put $c_U=\alpha+\rho h(U)$ and
\begin{equation}
\delta_K=r\big(|c_U|T_H(R)+3\rho c_H(R)L_H(R)\big).
\label{eq:hessian_variation_radius}
\end{equation}
Then for every $W\in\mathcal C$,
\begin{equation}
\lVert\mathsf K_M(W)-\mathsf K_M(U)\rVert_{\mathrm{op}}\le\delta_K,
\qquad
(m_U-\delta_K)I\preceq\mathsf K_M(W)
\preceq(L_U+\delta_K)I.
\label{eq:uniform_hessian_bounds}
\end{equation}
If $\mu_{\mathrm{loc}}=m_U-\delta_K>0$, set
$L_{\mathrm{loc}}=L_U+\delta_K$ and choose
$0<\gamma<2/L_{\mathrm{loc}}$. The primal map contracts on
$\mathcal C$ with the spectral factor from
Corollary~\ref{cor:stage_contraction}. If in addition a decoded block
has uniform error at most $\varepsilon$ on $\mathcal C$ and
$\lVert\mathcal P(U)-U\rVert_F+\varepsilon\le(1-q)r$, then all
conclusions of Theorem~\ref{thm:stage_certificate} apply, with total depth $(2\lceil\log_2p_{\max}\rceil+8)L_{\mathrm{tol}}$.
\end{theorem}

\begin{proof}
The loss Hessian is constant. Write $K_h=\nabla_M^2h$ and
$c(W)=\alpha+\rho h(W)$. The Hessian difference in
\eqref{eq:full_masked_hessian} is
\[
c_U\big(K_h(W)-K_h(U)\big)
+\rho\big(h(W)-h(U)\big)K_h(W)
+\rho\big(H(W)\otimes H(W)-H(U)\otimes H(U)\big).
\]
Every segment from $U$ to $W\in\mathcal C$ is inside
$\mathcal B_R^M$. The first term is bounded by
$|c_U|T_H(R)\lVert W-U\rVert_F$. For the second, the mean value
bound $|h(W)-h(U)|\le c_H(R)\lVert W-U\rVert_F$ gives
$\rho c_H(R)L_H(R)\lVert W-U\rVert_F$. The rank-one difference
has norm at most
$(\lVert H(W)\rVert_F+\lVert H(U)\rVert_F)
\lVert H(W)-H(U)\rVert_F$, hence at most
$2c_H(R)L_H(R)\lVert W-U\rVert_F$ before multiplication by $\rho$.
Adding proves the operator norm bound. Applying it to unit allowed
directions proves both spectral bounds. The contraction corollary and
the separate invariance hypothesis then give the remaining conclusions.
\end{proof}

\begin{remark}[Point curvature and numerically computed Hessians]
The local test uses the centre eigenvalues \emph{and} the variation
radius $\delta_K$. Invariance of the chosen ball is a separate condition from a positive
eigenvalue at $U$. If the computed symmetric
masked Hessian $\widehat{\mathsf K}_M(U)$ has certified operator error
at most $\eta_U$, one may use
$m_U\ge\lambda_{\min}(\widehat{\mathsf K}_M(U))-\eta_U$ and
$L_U\le\lambda_{\max}(\widehat{\mathsf K}_M(U))+\eta_U$.
Certified upper bounds for the other constants and for
$\lVert\mathcal P(U)-U\rVert_F$ are also needed for a numerical
certificate. Merely evaluating these quantities in floating point
without error enclosures is an empirical diagnostic.
\end{remark}

\begin{proposition}[Verified interval Hessian certificate]
\label{prop:interval_stage_certificate}
Fix finite stage controls $\alpha\in\R$, $\rho,\lambda\ge0$, $\gamma>0$
and $b=0$. Let $U\in\mathcal W_M$, $r>0$, and
$\mathcal C=\{W\in\mathcal W_M:\lVert W-U\rVert_F\le r\}$.
Assume the number $d_M$ of allowed entries is positive and use one
coordinate per allowed entry. Let $[K^-,K^+]$ be a finite entrywise
enclosure of the actual masked Hessian $\mathsf K_M(W)$ for every
$W$ in the coordinate box $|W_{ij}-U_{ij}|\le r$ on allowed entries,
with forbidden entries fixed at zero. Let $Q\in\R^{d_M\times d_M}$
be any recorded matrix with finite entries for which a verified bound
gives $\lVert Q^\top Q-I\rVert_2\le e$, $0\le e<1$.
Enclose every $Q^\top KQ$, $K\in[K^-,K^+]$, entrywise by a finite
interval matrix $[B^-,B^+]$, and set
\[
\begin{split}
\underline b&=\min_i\left(B^-_{ii}-\sum_{j\ne i}
\max\{|B^-_{ij}|,|B^+_{ij}|\}\right),\\
\overline b&=\max_i\left(B^+_{ii}+\sum_{j\ne i}
\max\{|B^-_{ij}|,|B^+_{ij}|\}\right).
\end{split}
\]
If $\underline b>0$, the bounds
\[
\mu=\frac{\underline b}{1+e},\qquad
L_J=\frac{\overline b}{1-e}
\]
give $\mu I\preceq\mathsf K_M(W)\preceq L_JI$ uniformly on
$\mathcal C$. Thus the contraction and update bounds of
Theorem~\ref{thm:stage_certificate} apply whenever the recorded step
size gives $q=\max\{|1-\gamma\mu|,|1-\gamma L_J|\}<1$, a verified
centre displacement $\overline d\ge\lVert\mathcal P(U)-U\rVert_F$
and a uniform primal block error $\varepsilon_W$ satisfy
$\overline d+\varepsilon_W\le(1-q)r$, and the implementation
preserves the mask and valid persistent encoding. For the literal
block, the latter error and encoding hypotheses are supplied under the
kernel and range assumptions of Theorem~\ref{thm:block_rounding} and
Corollary~\ref{cor:precision_stage}. An approximate eigenvector matrix
may be used as $Q$; no unverified eigenvalue is used in this conclusion.
\end{proposition}
\begin{proof}
The Frobenius ball is contained in the stated coordinate box. Each actual
masked Hessian is symmetric, so $B(W)=Q^\top\mathsf K_M(W)Q$ is
symmetric. Its entries lie in $[B^-,B^+]$, and Gershgorin's theorem
therefore gives $\underline b I\preceq B(W)\preceq\overline b I$.
The Gram bound implies
$(1-e)I\preceq Q^\top Q\preceq(1+e)I$, hence $Q$ is invertible.
For $v=Qx$, compare $v^\top\mathsf K_M(W)v=x^\top B(W)x$
with $\lVert v\rVert^2=x^\top Q^\top Qx$. Positivity of
$\underline b$ gives the lower ratio $\underline b/(1+e)$ and the
upper ratio $\overline b/(1-e)$, proving the claimed Hessian bounds.
Corollary~\ref{cor:stage_contraction} and the separate displacement,
error and invariance tests give the remaining conclusions.
\end{proof}

\paragraph{Obtaining the interval matrix.}
For the finite polynomial constraint, put $Y=I+(W\odot W)/N$.
For allowed entries $(i,j)$ and $(a,b)$, the exact derivative formulas are
\[
\begin{split}
H_{ij}&=2W_{ij}(Y^{N-1})_{ji},\\
(K_h)_{ij,ab}&=2\mathbf 1_{(i,j)=(a,b)}(Y^{N-1})_{ji}\\
&\quad+\frac{4W_{ij}W_{ab}}N
\sum_{\ell=0}^{N-2}(Y^\ell)_{ja}(Y^{N-2-\ell})_{bi}.
\end{split}
\]
The trace derivative gives the first identity. Differentiating it with
$\partial Y/\partial W_{ab}=(2W_{ab}/N)E_{ab}$ and the product rule
for $Y^{N-1}$ gives the second. The constant loss Hessian has entries
$(\mathsf S_M)_{ij,ab}=\mathbf 1_{j=b}\Sigma_{ia}$.
Evaluating these finite sums, products and powers with outward interval
arithmetic, then using
\[
\mathsf K_M=\mathsf S_M+(\alpha+\rho h)K_h+\rho H\otimes H,
\]
gives the required entrywise enclosure. The value $h$ can be enclosed
without subtracting nearly equal traces by starting
$S_0=(W\odot W)/N$, repeatedly setting $S_{k+1}=2S_k+S_k^2$,
and taking $h=\tr S_s$ for $N=2^s$.
The verification must record operation order, outward arithmetic and
finite endpoints; evaluating these formulas in ordinary floating point
alone is not an interval certificate.

For the spectral step, an outward maximum absolute row-sum bound on the
symmetric matrix $Q^\top Q-I$ supplies $e$. The formulas for
$\underline b,\overline b,\mu,L_J$ and $q$ must also be enclosed
outward in a numerical implementation. For the centre displacement,
exact rational evaluation of the nominal binary inputs and the
soft-threshold branch, followed by an outward square-root bound,
supplies $\overline d$. These computations concern the recorded
finite-sample covariance and fixed controls, not the population graph
or the changing controller. If there are no allowed entries, the primal
space is a singleton and the spectral test is unnecessary.

For example, with $p=N=2$, $\lambda=0$, $\alpha=10$, $\rho=1$,
$U_{12}=U_{21}=1/5$ and $\Sigma_{12}=\Sigma_{21}=0.2800064$,
$\Sigma_{11}=\Sigma_{22}=1$, the centre is a fixed point. Its masked
Hessian eigenvalues are $0.599968$ and $2.200224$. At radius $0.01$,
$\delta_K<0.177$, so $\mu=0.42$, $L_J=2.38$ and $\gamma=5/7$
give $q=0.7$. The centre displacement is zero and the error allowance
is $0.003$. The global test is negative even with the sharper constants.
This illustrates a certified neighborhood of a cyclic fixed point;
it supplies no graph-recovery guarantee.

\section{Transformer Construction and Computational Cost}
\label{app:construction}
This appendix gives the weights and state layout used in
Figure~\ref{fig:polyprox-transformer} and Theorem~\ref{thm:realization},
including how temporary values are cleared before the next update.

\subsection{Where the State Is Stored}
\label{app:coordinates}
Put $P=p_{\max}$ and $d=13P+18$. Use zero-based coordinate indices.
For $r=0,\ldots,12$, the wide register $A_r$ occupies
$rP,\ldots,(r+1)P-1$, in the order
\begin{equation}
(A_0,\ldots,A_{12})=
(\mathsf{id},\Sigma,M,W,B,Q,S,\mathsf{SW},\mathsf{ST},
\mathsf{HS},X,V,W^{\mathrm n}).
\label{eq:audit_wide_indices}
\end{equation}
Scalar coordinate $13P+j$, $j=0,\ldots,17$, has the following meaning:
\begin{center}
\begin{tabular}{@{}clcl@{}}
\toprule
$j$ & Register & $j$ & Register\\
\midrule
0 & $\iota_{\rm var}$ & 9 & $r_\gamma=\gamma\rho$\\
1 & $\iota_{\rm ctl}$ & 10 & $r_b=b\rho$\\
2 & $\lambda$ & 11 & $d_i$\\
3 & $\alpha$ & 12 & $h$\\
4 & $\rho$ & 13 & $\varsigma$\\
5 & $\gamma$ & 14 & $h^+$\\
6 & $b$ & 15 & $\bar\gamma$\\
7 & $\tau=\gamma\lambda$ & 16 & $\bar\tau$\\
8 & $a_\gamma=\gamma\alpha$ & 17 & $\bar\varsigma$\\
\bottomrule
\end{tabular}
\end{center}
The flags and five input controls occupy the first seven scalar
coordinates. All eleven remaining scalar coordinates and all nine
wide scratch registers start at zero. Variable row $i$ has
$\iota_{\rm var}=1$, $\iota_{\rm ctl}=0$, identity $e_i$, and padded
rows of $\Sigma,M,W$. The control row has the reverse flags, the five
controls, and zero wide registers. The barred scalars are broadcasts
on variable rows, separate from the corresponding control-row
quantities; this prevents a reset from deleting an input control.

\subsection{Attention and Feedforward Weights}
\label{app:selectors}
For a wide register $A$, let $P_A\in\mathbb R^{d\times P}$ select its
columns, and $E_A=P_A^\top$. For a scalar register $r$, let $p_r$ be
its $d$-dimensional selector column, and $e_r=p_r^\top$.
Use the following fixed heads, all read from the same pre-attention
residual stream within a layer:
\begin{align}
\mathcal H(A,B;C\to T):\quad
\Theta_Q&=P_A,\quad\Theta_K=P_B,\quad
\Theta_V=P_CE_T,\label{eq:audit_head}\\
\mathcal S(r\to t):\quad
\Theta_Q&=p_{\rm ctl},\quad\Theta_K=p_{\rm var},\quad
\Theta_V=p_re_t,\label{eq:audit_sum}\\
\mathcal B(r_1,\ldots\to t_1,\ldots):\quad
\Theta_Q&=p_{\rm var},\quad\Theta_K=p_{\rm ctl},\quad
\Theta_V=\sum_kp_{r_k}e_{t_k}.\label{eq:audit_broadcast}
\end{align}
A value map may route to several target registers; for example
$P_S(E_S+E_Q)$ copies $S$ into both $S$ and $Q$. The first pattern
adds $AB^\top C$. Thus $B=\mathsf{id}$ gives $AC$, and
$A=C=\mathsf{id}$ gives $B^\top$. Control-row identity is zero,
and $\mathsf{id}_{\rm var}$ is the padded identity. The other two
patterns sum a per-variable scalar into the control row and broadcast
control scalars onto variable rows, respectively.

To add a fixed linear residual increment $ZL$, use the ReLU branch
\begin{equation}
U_1=[L,-L],\qquad U_2=[I_d;-I_d],
\qquad \relu(ZU_1)U_2=ZL.
\label{eq:audit_linear_ff}
\end{equation}
Zero target columns can be omitted, so this needs at most $2d$ ReLU
units. For a product of source coordinates $u,v$, use one bilinear
unit with $G_1=p_u$, $G_2=p_v$. Its row in $G_3$ is the desired fixed
output routing vector. In particular one square unit can write both
$B_k=W_k^2$ and $S_k=W_k^2/N$. Several products concatenate their
selector columns. A residual overwrite $T\leftarrow f$ includes the
linear increment $-T$ and the routed expression $f$, evaluated from
the same pre-FF stream. Soft-thresholding uses the two preactivations
$V_k-\bar\tau$ and $-V_k-\bar\tau$, routed to $W^{\mathrm n}_k$ with
coefficients $1,-1$. Here $W^{\mathrm n}$ is still zero on entry to
the threshold layer.

\subsection{What Each Layer Computes}
\label{app:schedule}
The following list specifies every head and every product. Linear
resets and assignments use \eqref{eq:audit_linear_ff}; unlisted
registers are preserved. Attention precedes FF in each complete layer.
\begin{enumerate}
\item Initialize $B,S,Q$ and the four scalar products. There is no
attention. The FF products are $W_kW_k$ ($P$ units), routed to
$B_k,S_k$ with coefficients $1,1/N$, and
$\gamma\lambda,\gamma\alpha,\gamma\rho,b\rho$ (four units), routed
to $\tau,a_\gamma,r_\gamma,r_b$. Reset $Q=0$.
\item Repeat $s$ times, with no FF: heads
$\mathcal H(S,\mathsf{id};S\to S)$,
$\mathcal H(Q,\mathsf{id};S\to Q)$, and
$\mathcal H(\mathsf{id},\mathsf{id};S\to S,Q)$ add
$S^2,QS,S$ to their specified targets. With the existing residuals
this is exactly $S\leftarrow2S+S^2$ and
$Q\leftarrow Q+S+QS$. All three heads read the old $S,Q$.
\item With no attention, overwrite $d_i$ by
$\sum_{k=1}^P\mathsf{id}_{ik}S_{ik}$, using $P$ bilinear units.
The control-row result is zero.
\item Use four heads:
$\mathcal S(d\to h)$,
$\mathcal H(\mathsf{id},Q;\mathsf{id}\to\mathsf{ST})$,
$\mathcal H(\Sigma,\mathsf{id};W\to\mathsf{SW})$, and
$\mathcal B(\gamma,\tau\to\bar\gamma,\bar\tau)$.
Their initially zero targets receive $h=\tr S$, $Q^\top$, $\Sigma W$
and the two broadcasts. The following FF has $2P+1$ products:
$M_k(\mathsf{SW}_k-\Sigma_k)$ into $X_k$,
$W_k\mathsf{ST}_k$ into $\mathsf{HS}_k$, and $r_\gamma h$ into
$\varsigma$, with the linear addition $a_\gamma$.
Thus $\varsigma=\gamma\alpha+\gamma\rho h$ is formed in the
control row only \emph{after} the sum head has supplied $h$.
\item Use $\mathcal B(\varsigma\to\bar\varsigma)$, then the FF
overwrite
$V=W-\bar\gamma X-2\bar\varsigma\mathsf{HS}$.
It uses $2P$ products. The multiplier-weighted assembly follows the
previous layer's computation of $\varsigma$.
\item With no attention or bilinear units, write
$W^{\mathrm n}=\relu(V-\bar\tau)-\relu(-V-\bar\tau)$,
using $2P$ ReLU units.
\item With no attention, overwrite $B,S$ with the square products
of $W^{\mathrm n}$, routed with coefficients $1,1/N$, and reset
$Q=0$. This uses $P$ bilinear units. Old $B,S,Q$ are removed by
their linear residual increments.
\item Repeat the same $s$ three-head layers on these new $S,Q$.
\item Overwrite $d$ by the new diagonal readout, again with $P$
bilinear units. This reset removes the first pass's old $d$.
\item Use $\mathcal S(d\to h^+)$, followed by one product $r_bh^+$
into $\alpha$. In the same FF sublayer add $W^{\mathrm n}-W$ to
the primal register and erase all nine wide and eleven scalar
scratch registers. The product reads $r_b,h^+$ before they are erased,
so the control row becomes $\alpha+b\rho h(W^{\mathrm n})$.
\end{enumerate}
There are eight single layers and two $s$-layer passes. This list
proves that every input to an operation has already been formed;
there is no sequential computation hidden within one FF branch.

\subsection{State Preservation, Width and Depth}
\label{app:invariants}
Only variable rows carry the wide matrices. Wide control-row scratch
remains zero, and only the designated scalar broadcasts are nonzero
on variable rows. Every padded \emph{matrix column} remains zero:
the identity, data, mask and initial primal state have this property,
and row products, transposes and coordinate products preserve it.
The allowed-edge mask is required only for the persistent primal
register and the masked registers $B,\mathsf{HS},X,V,W^{\mathrm n}$.
It is not imposed on $S,Q,\mathsf{SW},\mathsf{ST}$; those scratch
matrices may contain nonzero entries at forbidden primal positions.

At the threshold layer, $X$ is masked by its explicit product with $M$,
$\mathsf{HS}$ is masked by its product with $W$, and hence $V$ is
masked. Since $\bar\tau\ge0$ on variable rows, soft-thresholding
preserves those zeros, including all padding. The final layer clears
scratch exactly in real arithmetic, keeps $\Sigma,M,\mathsf{id}$,
flags and controls fixed, and writes $W^+,\alpha^+$ to persistent
coordinates. It therefore establishes the clean-input invariant for
the next block. In floating point, resets and overwrite errors are
analyzed by Appendix~\ref{app:rounding}, rather than assumed absent.

The indexed schedule has at most four heads per layer, key dimension
at most $P$, $2d=26P+36$ ReLU units, and
\begin{equation}
B_{\max}(P)=\max\{P+4,\,2P+1\}
\label{eq:audit_bilinear_count}
\end{equation}
bilinear units per FF sublayer. The first layer needs $P+4$ products;
the gradient-preparation layer needs $2P+1$. The other counts are
$P,2P,0,P,P,1$ as listed above. Consequently the bound $2P+1$ holds for $P\ge3$, and the general literal
schedule needs $5$ and $6$ products at $P=1,2$, where $2P+1$ would give
$3$ and $5$. The count is an upper count for this schedule; the
single-variable primal problem itself is trivial. The depth,
embedding width, head bound and asymptotic resource orders are unchanged.

Every selector entry, sign, reset coefficient and routing coefficient
has magnitude $1$ or $2$, except the square-to-$S$ coefficient $1/N$.
The ReLU linear realization copies these magnitudes and signs. Thus
all nonzero parameters have magnitudes in $\{1/N,1,2\}$ and depend on
$P$ alone, including the unused padded selector coordinates.

\subsection{Arithmetic Cost and Parameter Storage}
\label{app:cost}
For $p\le P$, there are $p+1$ active tokens and $d=O(P)$ coordinates.
Each complete layer has a constant number of $O(P)$-wide
projections and at most four heads. Dense projection/product kernels
use at most $O(P^3)$ operations at the largest served size, giving
$O(P^3\log P)$ per update for $P\ge2$. Working activations can be
streamed with $O(P^2)$ storage; retaining every layer's activations
would instead use $O(P^2\log P)$.

Dense untied parameter matrices occupy $O(P^2\log P)$ numbers.
The explicit selector representation has $O(P)$ nonzero coefficients
per layer, and hence $O(P\log P)$ in the full stored schedule.
The identical doubling layers can also share one parameter template.
An $O(P^2)$ working-memory statement therefore does not include
materializing separate dense parameters or saving all activations.
These counts are arithmetic/number-storage bounds, not finite-bit
precision bounds. For $P=1$ the counts are constant.

\section{Why Exact Execution Needs Logarithmic Depth}
\label{app:depth_lower}
One counted layer is a residual linear-attention sublayer followed by a
residual bilinear/ReLU sublayer as in \eqref{eq:ff}. All heads within a
sublayer read the same input. Width, head count and token count are
finite and fixed along the scalar input family below; parameters are
independent of that input. Fixed affine biases or output projections
are allowed. Encoding and readout are affine, with no external writes
inside the block. Softmax, normalization and division by a variable
are outside this class.

\paragraph{How polynomial degree grows with depth.}
If the input coordinates are polynomials of degree at most $d$, the
queries, keys and values have degree at most $d$, and each attention
output has degree at most $3d$. On an interval where all ReLU signs are
fixed, the ReLU branch does not raise this degree, while the bilinear
branch raises it to at most $6d$. Residual sums and affine maps do not
increase these bounds. Each nonzero polynomial preactivation has
finitely many roots; identically zero ones add no boundaries. A finite
network therefore has a finite partition into open intervals and
boundary points, with polynomial output degree at most $6^D$ on each
open interval for affine input. Increasing finite width or head count
changes the partition, but not the degree bound.

\paragraph{Inputs that require high-degree outputs.}
Let $p=2$, $A=\bigl(\begin{smallmatrix}0&1\\1&0\end{smallmatrix}\bigr)$,
$M=A$, $W(t)=tA$, $\Sigma=I_2$, and
$\alpha=0$, $\rho=\gamma=b=1$, with any fixed $\lambda\ge0$.
The exponent $N=2^{\lceil\log_2p_{\max}\rceil}$ is even for
$p_{\max}\ge2$. With $a=t^2/N$, the eigenvalues of $Y=I_2+aA$ are
$1+a$ and $1-a$. Thus
\[
\psi_N(t):=h(W(t))=(1+a)^N+(1-a)^N-2,\qquad
\kappa_N(t):=H_{12}(W(t))=H_{21}(W(t))
=t\big[(1+a)^{N-1}-(1-a)^{N-1}\big].
\]
Their degrees are $2N$ and $2N-1$, with leading coefficients
$2/N^N$ and $2/N^{N-1}$. Hence $q_N=\psi_N\kappa_N$ has degree
$4N-1$, positive leading coefficient $4/N^{2N-1}$ and tends to
infinity as $t\to+\infty$. Choose an open interval with $q_N(t)>\lambda$.
The allowed loss-gradient entries are $t$, so the pre-threshold
entries are $-q_N(t)$, on the negative active branch. Therefore
\[
W^+(t)=w_N(t)A,\quad w_N(t)=\lambda-q_N(t),\qquad
\alpha^+(t)=\psi_N(w_N(t)).
\]
The primal degree is $4N-1$ with nonzero leading coefficient
$-4/N^{2N-1}$. The multiplier degree is $2N(4N-1)$ with leading
coefficient $2\,4^{2N}/N^{N(4N-1)}>0$. Lower-degree terms cannot
cancel either coefficient.

\begin{proof}[Proof of Theorem~\ref{thm:depth_lower}]
On a nonempty open subinterval of the active family, the network's
affine readout is polynomial of degree at most $6^D$. Exactness on
that interval implies polynomial identity with $w_N$ for the primal
output, and with $\psi_N\circ w_N$ for the multiplier. Thus
$6^D\ge4N-1$ for primal execution and
$6^D\ge2N(4N-1)$ for full-state execution, proving the integer bounds.
Since $p_{\max}\le N<2p_{\max}$, the upper construction and lower
bound both have order $\log p_{\max}$.
\end{proof}

The hard family has two variables but uses the specified exponent $N$;
this is not a lower bound for every two-variable DAG solver. Uncounted
nonlinear preprocessing or readout can change the degree argument.
Approximation, exactness on a finite test set, and normalized attention
are not covered. For $N=2$, $\lambda=0$, the outputs reduce to
$w_2=-t^7/2$ and $\alpha^+=t^{28}/32$, verifying both degrees directly.

\section{Error from Losing Multiplier Information}
\label{app:state_information}
Fix the data, mask, primal state and controls, and let
$T(\alpha)=W^+(\alpha)$ and $H=\nabla h(W)$. On every fixed active
branch in $J$, soft-thresholding subtracts a constant signed threshold.
Thus its output difference on $J$ is
$-\gamma H_J(\alpha_1-\alpha_2)$. Nonexpansiveness of the mask and
soft-threshold gives the complementary upper bound:
\[
g_J|\alpha_1-\alpha_2|
\le\lVert T(\alpha_1)-T(\alpha_2)\rVert_F
\le\gamma\lVert H\rVert_F|\alpha_1-\alpha_2|.
\]
The lower bound needs the active-branch hypothesis; the upper bound
holds for every pair of multipliers.

\begin{proof}[Proof of Proposition~\ref{prop:state_information}]
If the encoder merges two multipliers, their common prediction cannot
be within less than half their output separation of both targets.
This proves the paired-state claim. If a decoder is $L_D$-Lipschitz
and each prediction has error at most $\varepsilon$, the triangle
inequality through both predictions gives
\[
g_J|\alpha_1-\alpha_2|\le2\varepsilon+
L_D\lVert E(\alpha_1)-E(\alpha_2)\rVert.
\]
For an encoder taking at most $K$ values, each nonempty fibre therefore
has diameter at most $2\varepsilon/g_J$. Its closed interval hull has
that length, and at most $K$ such hulls cover $I$ of length $A$.
Consequently $A\le2K\varepsilon/g_J$. This requires neither measurable
nor contiguous fibres. Substituting $K=2^B$ gives the bit bound.
\end{proof}

\paragraph{When the lower bound is attained.}
Partition $I$ into $K$ equal bins and decode each bin to the exact
update at its midpoint. The upper separation bound gives uniform
error at most $\gamma\lVert H\rVert_FA/(2K)$. When all nonzero-gradient
entries remain on fixed active branches, $g_J=\gamma\lVert H\rVert_F$,
so this attains the lower bound. For instance, at $p=N=2$,
$W=t\bigl(\begin{smallmatrix}0&1\\1&0\end{smallmatrix}\bigr)$,
$\Sigma=I$, $\gamma=\rho=1$, $\lambda=0$ and $\alpha\in[0,A]$,
$T(\alpha)=-(\alpha+t^4/2)t^3
\bigl(\begin{smallmatrix}0&1\\1&0\end{smallmatrix}\bigr)$ for $t>0$.
Both branches are active and $g_J=\sqrt2t^3$.
The bits count the whole multiplier-dependent code; the claim does
not prescribe a register layout or an average-case learning rate.

\section{Rounding Bounds for the Constructed Block}
\label{app:rounding}
We bound rounding errors through the operations in
Appendix~\ref{app:construction}, including both constraint evaluations
and all changes to the stored state. The bound is computed from the
weights and input bounds, so it applies beyond the sampled executions.
It may be conservative and depends on the specified arithmetic
implementation with fixed controls. Section~\ref{sec:stage} uses it to
bound the number of updates needed for a target accuracy.

\subsection{Arithmetic and inputs}
\label{app:rounding_model}
Let $u$ be the unit roundoff. For each scalar addition, subtraction and
multiplication assume
\begin{equation}
\operatorname{fl}(a\mathbin{\circ}b)
=(a\mathbin{\circ}b)(1+\delta)+\eta,
\qquad |\delta|\le u,\quad |\eta|\le\omega,
\label{eq:rounding_model}
\end{equation}
with finite intermediates, exact negation, and exact zero operations
(zero products/results and addition of zero). Stored controls satisfy
$\lambda,\rho\ge0$, $\gamma>0$, and the stored threshold is nonnegative.
The allowance $\omega\ge0$ covers a specified underflow mode; setting
$\omega=0$ requires the relative model to hold for every nonzero
intermediate. A flush-to-zero kernel needs its own valid $\omega$.
ReLU is exact on stored values. Matrix products are classical dot
products in the recorded order, or kernels with the same proved bound.
Unmodeled casts, reduced-precision products, or reassociation require a
new graph or a replacement kernel bound. No softmax or normalization is
part of this architecture. The dyadic weights $0,\pm1/N,\pm1,\pm2$,
identity, flags, mask, and $b$ are exactly represented; otherwise their
conversion errors must be included as inputs.

For a dot product of length $k$, define
\begin{equation}
\Gamma_k=\frac{ku}{1-ku},\qquad
\zeta_k=2k(1+u)^k\omega,\qquad ku<1.
\label{eq:rounding_dot_constants}
\end{equation}
Entrywise absolute values and inequalities are used below. A classical
matrix product obeys
\begin{equation}
|\operatorname{fl}(AB)-AB|
\le\Gamma_k|A||B|+\zeta_k\boldsymbol 1,
\label{eq:rounding_dot}
\end{equation}
where $k$ is its inner dimension and $\boldsymbol 1$ has its output
shape. Products on the right are exact nonnegative matrix products.
The usual relative dot-product proof bounds the product of at most $k$
rounding factors by $\Gamma_k$. There are at most $2k-1$ additive
rounding noises from products and sums, each amplified by at most
$(1+u)^k$; this proves the stated, slightly loose, underflow allowance.
Exactly zero or monomial selector products may instead be treated as
exact when that fact follows from their operands and the kernel.

The nominal state is $z=(W,\alpha)$ with the fixed covariance and
controls of the intended real-arithmetic update. The encoded input has
clean scratch. Input bounds include covariance/control conversion when
the intended values differ from the stored ones. To compare one update
at its current decoded floating-point $W$, that $W$ is the nominal
input, so its input error is zero; the difference from an earlier exact
trajectory is handled by Theorem~\ref{thm:tracking}, not counted again
as a new conversion error.

\subsection{Bounding the Error in Each Layer}
\label{app:rounding_enclosure}
For a nominal graph node $A$ and its computed value $\widehat A$, carry
two nonnegative arrays
\[
|A|\le B_A,\qquad |\widehat A-A|\le E_A.
\]
Constants have $B_A=|A|$, $E_A=0$ when exact. Given valid input bounds,
the following rules propagate them. Write
\begin{align*}
T_{AB}&=B_AE_B+E_AB_B+E_AE_B,\\
U_{AB}&=(B_A+E_A)(B_B+E_B).
\end{align*}
For a matrix product $C=AB$ of inner dimension $k$, use
\begin{equation}
B_C=B_AB_B,\qquad
E_C=T_{AB}+\Gamma_kU_{AB}+\zeta_k\boldsymbol 1.
\label{eq:enclosure_matmul}
\end{equation}
For $C=A\pm B$, use
\begin{equation}
\begin{split}
B_C&=B_A+B_B,\\
E_C&=E_A+E_B+u(B_A+B_B+E_A+E_B)
+\omega\boldsymbol 1.
\end{split}
\label{eq:enclosure_add}
\end{equation}
For an entrywise product $C=A\odot B$, use
\begin{equation}
\begin{split}
B_C&=B_A\odot B_B,\\
E_C&=B_A\odot E_B+E_A\odot B_B+E_A\odot E_B\\
&\quad+u(B_A+E_A)\odot(B_B+E_B)+\omega\boldsymbol 1.
\end{split}
\label{eq:enclosure_hadamard}
\end{equation}
ReLU permits $B_C=B_A$, $E_C=E_A$, since it is 1-Lipschitz.
Transpose, exact copy, and exact negation permute or retain these
bounds. A cast with parameters $u_c,\omega_c$ permits
$B_C=B_A$ and
$E_C=E_A+u_c(B_A+E_A)+\omega_c\boldsymbol 1$.
Use the actual signs and exact-zero structure to tighten these rules
where justified, rather than replacing a literal operation by an
algebraically different solver replay.

\paragraph{Expanding each layer into arithmetic operations.}
For each head expand
\[
A_Q=Z\Theta_Q,\quad A_K=Z\Theta_K,\quad A_V=Z\Theta_V,
\qquad O=(A_QA_K^\top)A_V.
\]
Bound the three projections, the score product and the value product
in that order. Sum the heads and add the residual in the recorded
order. From that attended state bound $ZU_1$, its ReLU, and the
product with $U_2$; separately bound $ZG_1$, $ZG_2$, their entrywise
product and its product with $G_3$. Sum the two feedforward branches
and add the residual in the recorded order. These are precisely the
linear-attention and bilinear/ReLU layers of
Section~\ref{sec:realization}. The same rules cover both squaring
passes, the loss product $\Sigma W$, scalar assembly and broadcasts,
the two threshold preactivations, the old-to-new scratch overwrites,
and the final $W,\alpha$ writeback. Inner dimensions are read from
these products; in particular, a dense projection's dimension must
not be silently replaced by $p$.

\begin{proof}[Proof of Theorem~\ref{thm:block_rounding}]
For multiplication,
$\widehat A\widehat B-AB=A(\widehat B-B)
+(\widehat A-A)B+(\widehat A-A)(\widehat B-B)$.
Taking entrywise absolute values gives $T_{AB}$. Equation
\eqref{eq:rounding_dot} bounds the arithmetic error at the computed
operands by $\Gamma_kU_{AB}+\zeta_k\boldsymbol 1$.
The scalar arithmetic model proves the sum and entrywise-product
rules. ReLU's Lipschitz property proves its rule even at a threshold
crossing. Induction over the finite directed computation graph proves
every enclosure, including all residual additions. Extracting the
output coordinates gives
\begin{equation}
\varepsilon_W=\lVert E_{W,\mathrm{out}}\rVert_F,\qquad
\varepsilon_\alpha=E_{\alpha,\mathrm{out}},\qquad
\varepsilon_*=\max\{\varepsilon_W,\varepsilon_\alpha\}.
\label{eq:enclosure_output}
\end{equation}
Thus the decoded state error in the product norm of
Section~\ref{sec:promptmatters} is at most $\varepsilon_*$.

For a uniform certificate on a specified region, choose uniform input
$B,E$ and apply these monotone nonnegative rules once. For example, on
$\mathcal C=\{W\in\mathcal W_M:\lVert W-U\rVert_F\le r\}$ one may
use $B_W=|U|+rM$, or the looser $B_W=(\lVert U\rVert_F+r)M$,
with fixed covariance and controls bounded explicitly. This bounds
all representable decoded inputs in $\mathcal C$; include input
rounding if the implementation first converts arbitrary real inputs.
Compute the bounds exactly or with outward rounding. A rounded
estimate of an enclosure is not itself a certified upper bound.

With fixed finite input bounds, each $E$ tends to zero as
$u,\omega$ and the input/weight error bounds tend to zero: this follows
by induction from the formulas, $\Gamma_k\to0$, and $\zeta_k\to0$.
This continuity assertion presumes an admissible arithmetic family
with sufficient range. To justify the finite-intermediate hypothesis,
also check the range of products and partial sums inside each kernel,
using absolute-sum majorants or the scalar-expanded graph. Bounded
mathematical inputs alone do not exclude implementation overflow.
\end{proof}

\paragraph{Passing the state to the next block.}
Immutable input columns, identity, flags and mask receive no writes and
are preserved exactly by the selector construction. Forbidden entries
of $W$ remain exactly zero: the loss term is multiplied by $M$, the
cycle term is multiplied by $W$, their pre-threshold combination is
zero there, and the two ReLU branches give $\soft_\tau(0)=0$ for the
stored $\tau\ge0$. Zero padding is preserved for the same reason.
Choose the final scratch clears as pure coordinate negations followed
by the residual addition: $\operatorname{fl}(x+(-x))=0$ for finite
stored $x$. They are exact, so the next block starts with clean
scratch. At $b=0$, the multiplier increment is exactly zero and
$\alpha$ stays fixed. All these statements require finite operands;
in particular the second pass still runs when $b=0$, and $0\cdot
\mathrm{Inf}$ cannot be used as a zero increment. If a different
implementation perturbs supposedly immutable columns or does not
clear scratch exactly, enclose and carry those extra state coordinates
between blocks instead of claiming the two-register state invariant.

\subsection{Replacing Stored Values and Recomputing the Constraint}
\label{app:rounding_writeback}
A new value in a residual stream is generally written by subtracting
the old one and adding the residual. This is not an exact assignment
in floating point and must appear in the graph enclosure.

\begin{lemma}[Residual writeback]
\label{lem:rounding_writeback}
For a finite stored vector $x$, a target $y$, and a computed target
$\widehat y$ with $\lVert\widehat y-y\rVert_F\le e$, suppose the
writeback is $\widehat y_{\mathrm{out}}
=\operatorname{fl}(x+\operatorname{fl}(\widehat y-x))$.
For $m$ written coordinates under \eqref{eq:rounding_model},
\begin{equation}
\begin{split}
\lVert\widehat y_{\mathrm{out}}-y\rVert_F
\le{}&(1+u)^2e+u(1+u)\lVert y-x\rVert_F\\
&+u\lVert y\rVert_F+(2+u)\sqrt m\,\omega.
\end{split}
\label{eq:rounding_writeback}
\end{equation}
\end{lemma}
\begin{proof}
Let $t=\operatorname{fl}(\widehat y-x)$. Componentwise,
$|t-(\widehat y-x)|\le u|\widehat y-x|+\omega\boldsymbol 1$.
The next addition then gives
$|\widehat y_{\mathrm{out}}-\widehat y|
\le u(1+u)|\widehat y-x|+u|\widehat y|
+(2+u)\omega\boldsymbol 1$.
Use $|\widehat y-x|\le|y-x|+|\widehat y-y|$ and
$|\widehat y|\le|y|+|\widehat y-y|$, and take Frobenius norms.
The coefficient of $e$ is $1+u(1+u)+u=(1+u)^2$.
\end{proof}

For example, binary32 round-to-nearest with $x=1$, $y=2^{-30}$
gives $\operatorname{fl}(y-x)=-1$ and
$\operatorname{fl}(x-1)=0$. Binary64 has the same example with
$y=2^{-60}$. The intermediates are normal and finite. Thus a
relative-error bound depending only on the new value cannot cover
general residual writeback.

In the block, this applies to $W\leftarrow W^{\mathrm n}$, to the
second initialization $S\leftarrow (W^{\mathrm n}\odot
W^{\mathrm n})/N$, and to overwritten diagonal scratch. The old
$S$ is the first-pass polynomial value and need not be small relative
to the new $S_0$. The first-pass relative model therefore cannot simply be restarted.
Use its actual absolute initialization error and
the full graph rules \eqref{eq:enclosure_matmul}--\eqref{eq:enclosure_hadamard}
through the second pass, trace, multiplier increment and final
writeback. Clearing $Q$ alone does not remove these other errors.
These terms are why Theorem~\ref{thm:block_rounding} analyzes the
literal residual layers rather than replacing the reset and assignment
by exact operations in an arithmetic replay.

\subsection{Choosing Precision for a Target Accuracy}
\label{app:rounding_stage}
\begin{corollary}[Precision needed for a target accuracy]
\label{cor:precision_stage}
Fix the stage controls with $b=0$ and a ball $\mathcal C$ with uniform
curvature and contraction factor $q<1$ from Section~\ref{sec:stage}.
Obtain $\varepsilon_W(u,\omega)$ from Theorem~\ref{thm:block_rounding}
using uniform input bounds on $\mathcal C$, including any covariance and
control conversion errors, and verify the range of \emph{all} operations
in both passes. Use the exact mask, padding and scratch invariants of
Appendix~\ref{app:rounding}. With $d=\lVert\mathcal P(U)-U\rVert_F$, a
sufficient precision test for an invariant ball and target
$\mathrm{tol}$ is
\begin{equation}
\varepsilon_W(u,\omega)\le(1-q)r-d,\qquad
\varepsilon_W(u,\omega)<(1-q)\mathrm{tol}.
\label{eq:precision_selection}
\end{equation}
Then Theorem~\ref{thm:stage_certificate} and \eqref{eq:integer_budget}
apply with $\varepsilon=\varepsilon_W$: the error floor is
$\varepsilon_W/(1-q)$ and each update uses $2\lceil\log_2p_{\max}\rceil+8$
transformer layers. If the invariance slack is strictly positive and
conversion errors tend to zero, sufficiently accurate members of an
arithmetic family with adequate range meet this test. No fixed dtype is
asserted to work for every input or stage.
\end{corollary}

\begin{proof}[Proof of Corollary~\ref{cor:precision_stage}]
The uniform input enclosure on $\mathcal C$, with the stated kernel
and range checks, supplies a bound $\varepsilon_W$ for every
implemented transition at a representable decoded input in that ball.
The exact-zero properties above give the mask and clean-scratch
invariants; $b=0$ preserves the fixed multiplier exactly.
The radius argument of Theorem~\ref{thm:stage_certificate} is
$\lVert\widehat{\mathcal P}(W)-U\rVert_F
\le\varepsilon_W+q\lVert W-U\rVert_F+d\le r$.
It therefore applies at each decoded iterate. The error recurrence
and sufficient update count follow from that theorem with
$\varepsilon=\varepsilon_W$.

If $d<(1-q)r$ and the target tolerance is positive, there is a positive
error allowance satisfying both inequalities in
\eqref{eq:precision_selection}. The enclosure tends to zero with
arithmetic and conversion errors for fixed finite bounds. An admissible
family with sufficient range therefore has sufficiently accurate
members meeting that allowance. The corollary is an existence statement together with a computable
selection test, and each dtype and kernel still needs its own range
audit. A small $u$ cannot compensate for an insufficient exponent range,
and changes made by the outer controller lie outside the certificate.
\end{proof}

\section{Experimental Details and Additional Results}
\label{app:experiments}
This appendix gives the protocols and additional results for
Section~\ref{sec:experiments}. Measurements from repeated blocks,
snapshots and different radii are dependent. We count nonfinite results
as failures and state when averages exclude them.

\subsection{Single Updates and Repeated Execution}
\label{app:executor}
The $254$ finite-reference checks use the legacy \texttt{executor.run}
wrapper, which freshly encodes each input, executes attention/FF layers,
and decodes $(W^+,\alpha^+)$. Capacities are $4,8,16,32$, with all
$p\le p_{\max}$, three mask densities, varied scalar controls and both
values of $b$. The independent reference uses matrix powers. Maximum
relative errors are $3.0\cdot10^{-15}$ in $W$ and $1.1\cdot10^{-13}$
in $\alpha$. One additional reference case overflows in the second
constraint evaluation: a large step sends a finite input to
$\lVert W^+\rVert_{\max}\approx3.4\cdot10^6$ at $N=32$.

The indexed construction in Appendix~\ref{app:construction} is checked
separately at capacities $1,2,4,8,16$, ten input seeds each. Its $50$
one-block comparisons have maximum Frobenius $W$ error
$9.56\cdot10^{-17}$ and zero multiplier error. Three-block reuse of the
full output encoding gives $150$ comparisons, with maximum $W$ error
$1.37\cdot10^{-16}$ and multiplier error $2.78\cdot10^{-17}$; the
maximum discrepancy from a clean encoding is $5.56\cdot10^{-17}$.
These are sampled float64 checks, not uniform certificates.

\paragraph{Why repeated execution needs a separate test.}
Direct reuse of the legacy implementation's full residual stream leaves
scratch scalars that contaminate later blocks. A three-block audit has
maximum errors $1.43\cdot10^{-4}$ in $W$ and $0.101$ in $\alpha$;
an explicit final-scratch reset reduces them to $1.52\cdot10^{-16}$
and $1.53\cdot10^{-14}$. The indexed specification includes these
clears. The historical $180$-case bridge to arithmetic replay reports
finite full states in $178$ cases and two failures on re-audit. We
therefore use indexed full-stream comparisons and the verified serial
continuations below as persistence evidence. Fresh-encoding agreement
and arithmetic replay alone do not establish that property.

\subsection{Rounding Errors at Different Precisions}
\label{app:new_precision}
The grid uses $(p,p_{\max})=(2,2),(3,4),(5,8)$, five input seeds,
input scales $0.01,0.1,0.5$, radius $0.01$, and float16/32/64:
$135$ cases in total. Evaluation is serial, unfused and left-to-right;
exact rational arithmetic at the nominal binary64 inputs supplies the
reference. Point and uniform-ball enclosures include input conversion.
All sampled executions are finite and below their point-error envelopes.
Maximum $W$ errors are $7.03\cdot10^{-4}$, $1.68\cdot10^{-7}$ and
$2.14\cdot10^{-16}$ for float16/32/64, respectively; multiplier maxima
are $8.39\cdot10^{-4}$, $2.52\cdot10^{-7}$ and $3.14\cdot10^{-16}$.

The uniform range test passes $129/135$ cases. The six failures are
float16 at scale $0.5$: $(p,p_{\max})=(3,4)$ with seed $4$, and $(5,8)$
with seeds $0$--$4$. Their sampled intermediates remain finite, but
uniform majorants exceed $65504$. All reset/immutable-coordinate
discrepancies are zero. Median point-bound/error ratios are approximately
$6960$, $6270$ and $8270$ by dtype, so these sufficient bounds are
conservative. Neither sampled containment nor this serial arithmetic
certificate supplies a uniform guarantee for an unspecified kernel.

\subsection{Testing the Update Bound with Repeated Blocks}
\label{app:practical_t3}
The captured corpus has $30$ Gaussian linear-SEM instances: $p=3,5,10$,
ten seeds per dimension, $m=500$, Erd\H os--R\'enyi graphs of expected
degree two, and $\lambda=0.1$. Three snapshots per trajectory give
$90$ centres. Freezing recorded $\alpha,\rho,\gamma$ with $b=0$, neither
the sharp-global nor local-Hessian test certifies a ball at radii
$0.001,0.01,0.1$; the verified interval retest also gives $0/270$.
Six two-node control balls pass the analytic tests. Snapshots and radius
attempts are dependent measurements, not independent instances.

For each snapshot, take the last saved state $V$ with matching
$\alpha,\rho$, and polish that fixed objective by L-BFGS-B on split
positive/negative variables for at most $1500$ iterations. This proposes
centre $U$; the update bound excludes these preliminary optimization steps.
Test radii $0.1,0.03,0.01,0.003,0.001,0.0003,0.0001,0.00003,
0.00001,0.000003,0.000001$ in descending order. One policy retains
recorded $\gamma$; the other chooses a binary64 value near $2/(\mu+L_J)$,
which defines a new fixed map. Stop at the first passing radius and
retain failed attempts in the denominator.

Outward binary64 interval products enclose the Hessian and basis Gram
error; Gershgorin gives uniform curvature. Rational evaluation of the
nominal proximal map bounds the centre displacement $\overline d$.
The literal serial graph gives $\varepsilon_W$ and range bounds.
Acceptance requires $\mu>0$, $q<1$, finite enclosures,
$\overline d+\varepsilon_W\le(1-q)r$, and
$\varepsilon_W/(1-q)<10^{-6}$. The update counts are verified using outward-rounded
bounds on $q^L$. Table~\ref{tab:practical_t3_coverage} reports coverage.

\begin{table}[htbp]
\centering
\begin{tabular}{lrrrrrr}
\toprule
Protocol & \shortstack{Radius\\attempts} & \shortstack{Passing\\snapshots} & $p=3$ & $p=5$ & $p=10$ & \shortstack{Saved $V$\\inside}\\
\midrule
Original centres, recorded $\gamma$ & 270 & 0/90 & 0 & 0 & 0 & 0\\
Polished centres, recorded $\gamma$ & 725 & 46/90 & 21 & 12 & 13 & 25\\
Polished centres, new $\gamma$ & 636 & 67/90 & 30 & 20 & 17 & 27\\
\bottomrule
\end{tabular}
\caption{Verified interval fixed-stage tests. Each dimension has $30$
snapshots from ten SEM instances; snapshots and radius attempts are not
independent experimental units. The polished policies retain one passing
ball per successful snapshot. The last column counts saved matching-control
endpoints in that policy's passing ball, not automatically successful
continuations of the original adaptive controller. Only $2$ and $4$
original snapshot states, respectively, lie in the polished passing balls.}
\label{tab:practical_t3_coverage}
\end{table}
For starts at half the ball radius, the sufficient number of updates
ranges from $43$ to $379087$ with the recorded step size, and from $24$
to $269405$ with the new one. The largest passing error floors are
$8.60\cdot10^{-8}$ and $6.17\cdot10^{-8}$, respectively. Passing the test
can therefore still require many updates. Failing it does not establish
that the map is noncontractive.

\begin{table}[htbp]
\centering
\begin{tabular}{lrrrrr}
\toprule
Starting policy & $p$ & Seed & Snapshot & \shortstack{Certified\\updates} & \shortstack{Final verified\\distance bound}\\
\midrule
Seeded half-radius perturbation & 3 & 5 & 0 & 134 & $7.41\cdot10^{-16}$\\
Seeded half-radius perturbation & 5 & 0 & 0 & 43 & $6.54\cdot10^{-14}$\\
Seeded half-radius perturbation & 10 & 3 & 0 & 978 & $1.83\cdot10^{-12}$\\
Saved matching-control endpoint & 3 & 8 & 0 & 186 & $3.43\cdot10^{-12}$\\
Saved matching-control endpoint & 5 & 9 & 0 & 426 & $5.97\cdot10^{-11}$\\
Saved matching-control endpoint & 10 & 3 & 0 & 546 & $2.51\cdot10^{-11}$\\
\bottomrule
\end{tabular}
\caption{Repeated execution of the constructed block in serial float64 arithmetic. All use the
recorded step size and $b=0$, with controls fixed throughout. Within each
starting policy and dimension, the selector prefers the recorded-step
policy and selects the smallest positive predicted update count, up to $5000$.
All six trajectories remain finite, stay inside their verified balls,
retain exactly clean scratch/immutable coordinates, and meet the
$10^{-6}$ fixed-point distance target by the predicted update count.
Final bounds use exact rational a posteriori residuals divided by $1-q$,
or the verified triangle bound through the centre, whichever is smaller.}
\label{tab:practical_t3_rollouts}
\end{table}
Each continuation feeds back the full residual stream and checks ball
membership and clean encoding at every transition. Perturbation starts
use a seeded masked direction at half-radius; endpoint starts encode
$V$ once. Both use $B_0=\lVert W_0-U\rVert_F+\overline d/(1-q)$.
These six runs reach their targets within the predicted update counts.
They test these specific fixed-control cases; they do not validate all
passing balls or the adaptive controller. Their targets are stage fixed
points, which need not be acyclic or statistically correct.

\subsection{Additional diagnostics for the reference solver}
\label{app:solver_diagnostics}

We compare Poly-prox, the direct implementation of
Algorithm~\ref{alg:polyprox}, with Poly-LBFGS (same polynomial
constraint, L-BFGS optimizer) and NOTEARS (L-BFGS, matrix-exponential
constraint), which separate the effect of the constraint from that of
the optimizer, and with sort-regress, PC and GES. The grid covers
Erd\H os--R\'enyi and scale-free linear SEMs with expected degree $1$,
$2$ and $4$, Gaussian and Laplace noise, $p\le20$ and
$m\in\{10^2,10^3\}$, each on raw data and after rescaling columns to
unit variance; $m=10^4$ is used only at expected degree $2$, Gaussian
noise and $p\in\{5,10\}$.

Each grid cell has three input seeds; raw and standardized versions
share a data draw. Poly-prox starts at $W=0$, $\alpha=0$, $\rho=1$,
uses at most $12$ outer stages and $400$ proximal steps per stage, with
acyclicity tolerance $10^{-8}$ and stationarity tolerance $10^{-6}$.
The reference L-BFGS methods allow $40$ outer stages and $200$ inner
iterations, with acyclicity tolerance $10^{-8}$. All three use
$\lambda=0.1$ and threshold $0.3$. PC uses stable Fisher-$z$ tests at
level $0.05$; GES uses the BIC score. Sort-regress regresses each variable
on those with smaller marginal variance. The implementations are
\texttt{solver.py} and \texttt{baselines.py}; the grid runner is
\texttt{run\_reference.py}.
Greedy score search provides a classical alternative \citep{Chickering2002OptimalSI}.
Continuous constraints can create optimization difficulties
\citep{Kaiser2021UnsuitabilityON,Seng2024LearningLD}; convergence of
augmented-Lagrangian methods requires separate hypotheses
\citep{Ng2020OnTC}. Raw-scale performance can reflect variance ordering
\citep{Reisach2021BewareOT,Reisach2023ASS}, and comparisons depend on
optimization and evaluation choices \citep{Ng2023StructureLW}.
Regularized DAG scores can fail to recover the target even at the
population level \citep{Jin2024RevisitingDS}. Published biological
networks provide another evaluation setting \citep{Sachs2005CausalPN};
our retained recovery experiment uses synthetic observations on fixed
benchmark topologies.

Figure~\ref{fig:shd} gives the mean SHD at $m=10^3$. On raw data
varsortability averages $0.96$ and the continuous solvers and
sort-regress stay ahead of PC and GES; rescaling each column to unit
variance drops it to $0.52$ and makes sort-regress the worst method at
$p=20$, while PC and GES, unaffected by column scaling, do not move.
Poly-prox pays a small price for its optimizer, reaching a median
constraint value near $7\cdot10^{-6}$ against the $10^{-8}$ tolerance
both L-BFGS solvers meet, so its SHD sits a little above theirs, though
its thresholded output is acyclic in every raw simulation.

\begin{figure}[htbp]
\centering
\includegraphics[width=0.75\textwidth]{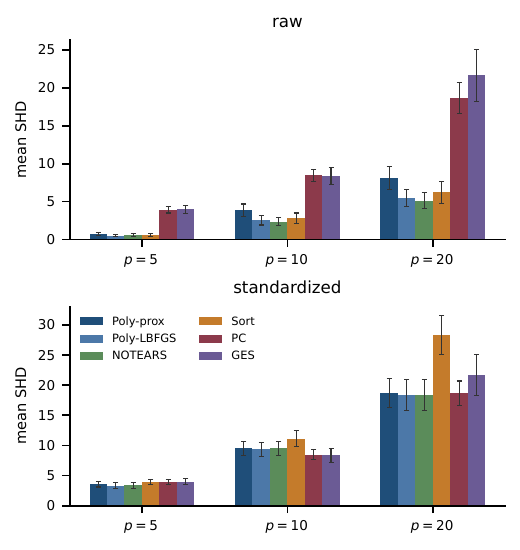}
\caption{\textbf{Accuracy of graph recovery.}
Lower structural Hamming distance (SHD) means fewer edge mistakes.
Each bar averages $36$ simulations with $1000$ observations; error bars
show standard errors. On raw data, the continuous solvers and
sort-regress outperform PC and GES. Rescaling each variable to unit
variance makes sort-regress the worst method at $20$ variables.
PC and GES are unchanged by this rescaling.}
\label{fig:shd}
\end{figure}

\subsection{Benchmark topology protocol}
\label{app:benchmark_protocol}
The recovery experiment keeps each benchmark's directed support and uses
one fixed weighted linear Gaussian SEM per topology, with edge magnitudes
in $[0.5,2]$ and random signs. Sample sizes are $m\in\{1000,5000,10000\}$ on every topology,
plus $m=500$ on CHILD. Noise is Gaussian with variance one; columns
are centered without standardization. The fixed topology order is
ASIA, CANCER, EARTHQUAKE, SURVEY, SACHS, CHILD, ALARM.
For zero-based topology index $i$, weight and data seeds are
\texttt{SeedSequence([1000+i,1])} and
\texttt{SeedSequence([2000+i,m,draw])}. Ten draws per sample size produce repeated observations of each fixed graph, not independent graph
weightings. The experiment uses direct arithmetic replay
under the reference solver's accepted controls and multiplier events;
it does not run the attention/FF residual graph for all datasets.

Across $220$ datasets, the final Frobenius and largest-entry weight gaps,
multiplier gap and thresholded graph disagreement are exactly zero in
float64. Figure~\ref{fig:inheritmaps} displays the median-SHD draw at
$m=10^4$. SHD is computed after thresholding at $0.3$.
The agreement establishes inheritance of these reference computations;
it does not certify recovery of the underlying graph or transfer to
other weightings of the same topology.

\subsection{Training Models with and without State Updates}
\label{app:learn}
We train ordinary tied-weight attention models on Poly-prox transition
labels, with five independent training runs per reported condition, and
compare their decoded outputs with the reference update. These models
use four tied attention/ReLU repeats; they do not instantiate the
full-depth bilinear architecture of Theorem~\ref{thm:realization}.
The experiments measure output agreement for these architectures and
budgets, rather than learnability of the guaranteed construction.
Table~\ref{tab:train} includes a collapsed-input model that withholds
$W,\alpha$, whereas Table~\ref{tab:memory} holds state access fixed and
changes residual write-back.

\paragraph{How we summarize runs and report failures.}
For one-step error, each training run is summarized by its median test
error, and tables report the mean and standard error of finite run
summaries. The trajectory metric is the unnormalized $12$-update
Frobenius error and is aggregated separately. All conditions have five
attempted training runs. A bracketed $[k]$ marks $k$ nonfinite run
summaries excluded from that cell's mean; ``non-finite (all)'' means all
five failed, and standard errors require at least two finite summaries. Finite-only means are conditional summaries,
not a success rate or an unconditional expected error.

\paragraph{Objective and data.} The independently drawn random-transition component uses a random DAG
on $p$ nodes, Erd\H os--R\'enyi or scale-free
with equal probability and expected degree from $\{1,2,4\}$;
$m\in\{200,500,1000\}$ observations from the linear SEM of
Section~\ref{sec:setup} with Gaussian or Laplace noise, each with
probability $\tfrac12$, and column standardization applied with
probability $\tfrac12$; a state $W$ with i.i.d.\ $\mathcal N(0,0.2^2)$
off-diagonal entries; and controls $\alpha\sim\mathrm{U}[0,2]$,
$\rho\in\{0.1,1,10\}$, $\gamma\sim\mathrm{U}[0.02,0.2]/\lVert
\Sigma\rVert_2$. The label is the exact Poly-prox update of that state.
Each corpus also contains states the solver actually visits: twelve
evenly spaced iterates from each of several short proximal runs, labelled
the same way, which is $960$ of the $2460$ transitions and $9600$ of the
$24600$. States from a shared solver run are dependent observations.
Random-transition labels use $b=1$, while solver-derived primal labels
use $b=0$. The ordinary models receive no explicit $b$ input, so their
multiplier supervision mixes two update policies. This does not alter
the primal target, but the shared training loss can still affect fitted
parameters; these experiments do not isolate why primal learning fails.
The loss is the mean squared error in the entries of
$\widehat W_{\ell+1}$ plus $0.05$ times the mean squared error in
$\widehat\alpha_{\ell+1}$, minimized with AdamW at learning rate
$3\cdot10^{-4}$, batch size $64$ and gradient-norm clipping at $1$.
Models are trained at $p=5$ with rows padded to $p_{\max}=10$ and
evaluated at $p=10$ without retraining. The five-seed runner is \texttt{experiments/run\_train\_seeds.py}, using
models from \texttt{run\_train.py}. Held-out tests are fixed across
training seeds: $200$ random transitions at $p=5$ (input seed $1$),
$120$ at $p=10$ (seed $2$), and $24$ twelve-step trajectories at $p=5$
(seed $3$). Training seeds vary initialization, training data and
minibatch order; the test corpus is shared.

Figure~\ref{fig:heatmaps} uses \texttt{plot\_heatmaps.py}: a five-node
instance with seed $22$ and a ten-node instance with seed $83$, both
with $800$ Gaussian observations, expected degree two and raw scale.
The displayed model outputs follow $12$ updates. The right column is
the reference solver's returned estimate, while red dots mark true
edges. These examples are illustrative, not averages.

\begin{table}[htbp]
\centering
\scriptsize
\resizebox{\columnwidth}{!}{
\begin{tabular}{@{}lccc@{}}
\toprule
Model & $p{=}5$ error & $p{=}10$ error (zero-shot) & trajectory error, 12 updates \\
\midrule
\multicolumn{4}{@{}l}{\emph{paper-size corpus}} \\
\quad softmax, with memory & $0.94\pm0.00$ & $0.97\pm0.00$ & $0.64\pm0.01$ \\
\quad linear, with memory & $0.98\pm0.01$ & $1.01\pm0.01$ & $24.24\pm22.50$ \\
\quad linear, no memory (collapsed input) & $10.48\pm0.04$ & $16.48\pm0.05$ & $1.1{\times}10^{3}\pm0.7{\times}10^{3}$ \\
\quad learned-threshold unroll & $3.94\pm0.06$ & $11.81\pm0.25$ & $2.8{\times}10^{2}\pm1.7{\times}10^{2}$ \\
\midrule
\multicolumn{4}{@{}l}{\emph{$\sim$10$\times$ corpus, 2.5$\times$ epochs}} \\
\quad softmax, with memory & $0.67\pm0.01$ & $1.56\pm0.02$ & $0.40\pm0.01$ \\
\quad linear, with memory & $0.98\pm0.01$ & $1.02\pm0.02$ & $1.26\pm0.23$\,[2] \\
\quad linear, no memory (collapsed input) & $10.40\pm0.07$\,[1] & $16.62\pm0.07$\,[1] & $2.3{\times}10^{22}\pm2.3{\times}10^{22}$\,[1] \\
\quad learned-threshold unroll & $0.72\pm0.01$ & $4.82\pm0.36$ & $0.41\pm0.01$ \\
\bottomrule
\end{tabular}
}
\caption{\textbf{Trained update executors, mean $\pm$ standard error over
five attempted training runs.} Step error is normalized so that copying the input scores
$1$; the median over test cases is taken within a seed. A bracketed
count gives the seeds whose value was non-finite, which are excluded from
that mean; trajectory means are dominated by a few divergent seeds.}
\label{tab:train}
\end{table}

\paragraph{Updating the current state or predicting it from scratch.}
The collapsed-input model in Table~\ref{tab:train} withholds $W,\alpha$
and predicts the next state from scratch. Its inputs do not determine
the target update. Table~\ref{tab:memory} therefore holds state access,
architecture, parameter count and training recipe fixed, changing only
whether the model predicts a residual update or the whole next state.
At the smaller corpus, removing write-back raises one-step errors from
$0.98$ to $10.44$ for linear attention and from $0.94$ to $10.98$ for
softmax. At the larger corpus, softmax closes part of this gap at $p=5$
($1.70$ versus $0.67$), but not at $p=10$ ($14.82$ versus $1.56$).
Because step error is normalized by a potentially small true update,
imperfectly reconstructing the incoming $W$ can incur a large score.
The comparison supports residual parameterization under this recipe;
it does not establish necessity of a particular state encoding.

\begin{table}[htbp]
\centering
\scriptsize
\resizebox{\columnwidth}{!}{
\begin{tabular}{@{}llccc@{}}
\toprule
Attention & Residual write-back & $p{=}5$ error & $p{=}10$ error & trajectory error \\
\midrule
softmax (paper-size) & persistent ($\widehat W=W+\delta$) & $0.94\pm0.00$ & $0.97\pm0.00$ & $0.64\pm0.01$ \\
softmax (paper-size) & absent (output from scratch) & $10.98\pm0.38$ & $17.01\pm0.39$ & $1.12\pm0.12$ \\
linear (paper-size) & persistent ($\widehat W=W+\delta$) & $0.98\pm0.01$ & $1.01\pm0.01$ & $24.24\pm22.50$ \\
linear (paper-size) & absent (output from scratch) & $10.44\pm0.05$ & $16.49\pm0.03$ & $1.37\pm0.22$ \\
\midrule
softmax (large corpus) & persistent ($\widehat W=W+\delta$) & $0.67\pm0.01$ & $1.56\pm0.02$ & $0.40\pm0.01$ \\
softmax (large corpus) & absent (output from scratch) & $1.70\pm0.13$ & $14.82\pm0.21$ & $0.65\pm0.03$ \\
linear (large corpus) & persistent ($\widehat W=W+\delta$) & $0.98\pm0.01$ & $1.02\pm0.02$ & $1.26\pm0.23$\,[2] \\
linear (large corpus) & absent (output from scratch) & $10.53\pm0.02$ & $16.48\pm0.06$ & $3.16\pm1.23$\,[3] \\
\bottomrule
\end{tabular}
}
\caption{\textbf{Residual write-back with identical inputs, five training runs.} ``Persistent'' writes the
update back into the state registers, $\widehat W=W+\delta$; ``absent''
gives the model the same inputs and the same parameter count but produces
$\widehat W$ and $\widehat\alpha$ from scratch. One-step errors are
normalized so that copying scores $1$; the trajectory column is the
unnormalized $\lVert\widehat W_{12}-W_{12}\rVert_F$ after twelve
unrolled updates. A bracketed count gives the seeds whose value was
non-finite, excluded from that mean. These are the numbers quoted in
Section~\ref{sec:experiments}.}
\label{tab:memory}
\end{table}

\subsection{Checking the Relative-Error Bound}
\label{app:exactplots}
Figure~\ref{fig:exact} connects the conditioning results to numerical
examples. Panel (1) compares direct arithmetic replay with an independent
matrix-power update on eight states per $p\in\{3,5,10,20\}$, at capacity
$p_{\max}=20$ and $N=32$. This is a replay check, separate from the
literal attention/FF tests in Appendix~\ref{app:executor}.
Panels (2)--(3) use the stationary two-node family in
Appendix~\ref{app:proof_multiplier}, with $r=0.6$, $\lambda=0.2$,
$\rho=0$ and exact $c=\alpha=(r-\lambda-t)/t^3$.
Multiplying $\nabla h$ by $1+\nu$, with $\nu=0.1$, keeps the assembled
gradient error below $(2\nu+\nu^2)C_k$ as in
Corollary~\ref{cor:relative_accuracy}. Because $\rho=0$, perturbing $h$
has no effect in this example: it tests relative gradient accuracy with
an exact coefficient. Adding $10^{-4}$ to each off-diagonal gradient
entry instead amplifies the error from $0.0035$ to $6.7$ as $c$ grows
from $25$ to $4.75\cdot10^4$. Panel (4) adds $10^{-3}$ to the multiplier
after each update at the origin, realizing the undamped drift in
Remark~\ref{rem:neutral}. The protocol and saved values are in
\texttt{experiments/run\_exact.py} and
\texttt{experiments/results/exact\_execution.json}.

\begin{figure}[htbp]
\centering
\includegraphics[width=0.85\textwidth]{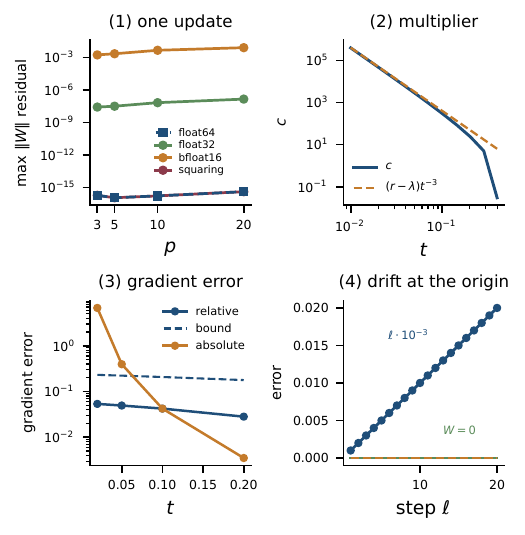}
\caption{\textbf{Arithmetic replay and conditioning checks.}
(1) One-update replay error across eight states per dimension.
(2) Effective-multiplier growth in the two-node stationary family.
(3) Relative gradient perturbations remain below the bound of
Corollary~\ref{cor:relative_accuracy}; fixed absolute perturbations are
amplified as $c$ grows. (4) Undamped multiplier error accumulates as
$\ell\cdot10^{-3}$ while $W=0$.}
\label{fig:exact}
\end{figure}

\clearpage
\subsection{How Stationarity Changes During Optimization}
\label{app:ex3_stationarity}
Figure~\ref{fig:ex3_stationarity} summarizes $30$ saved Gaussian linear-SEM
trajectories: ten input seeds for each $p\in\{3,5,10\}$, with $m=500$,
$\lambda=0.1$ and polynomial exponents $N=4,8,16$, respectively.
There are $44508$, $46590$ and $47282$ accepted primal outputs by
dimension ($138380$ total). These are arithmetic-replay states, not
literal-transformer activations. Each output uses its own recorded
pre-controller $\alpha,\rho,\gamma$; later controller changes can alter
$c=\alpha+\rho h$ and the augmented residual without changing $W$.

The masked-Lasso residual $d_M$ and the smallest augmented residual
$\eta_{\min}$ use the definitions in Appendix~\ref{app:proof_multiplier}.
For the relevant smooth gradient $g$, each allowed nonzero entry has
residual $g_{ij}+\lambda\operatorname{sign}(W_{ij})$; at an exactly zero
entry it is $\operatorname{sign}(g_{ij})(|g_{ij}|-\lambda)_+$.
Forbidden entries are excluded, and the Frobenius norm gives the
residual. No support threshold replaces exact zeros.
At the last primal output, median $\eta_{\min}/d_M$ is $0.00338$,
$0.00661$ and $0.0136$ by dimension. None of these endpoints has
$\eta_{\min}\le10^{-6}$ in absolute norm. The plots show finite
compensation of the Lasso residual by the cycle term, while changing
controls can increase the augmented residual. They do not demonstrate
$\eta_k\to0$, convergence to an acyclic non-Lasso limit, or multiplier
divergence, and do not certify the adaptive outer controller.

\begin{figure}[htbp]
\centering
\includegraphics[width=0.98\textwidth]{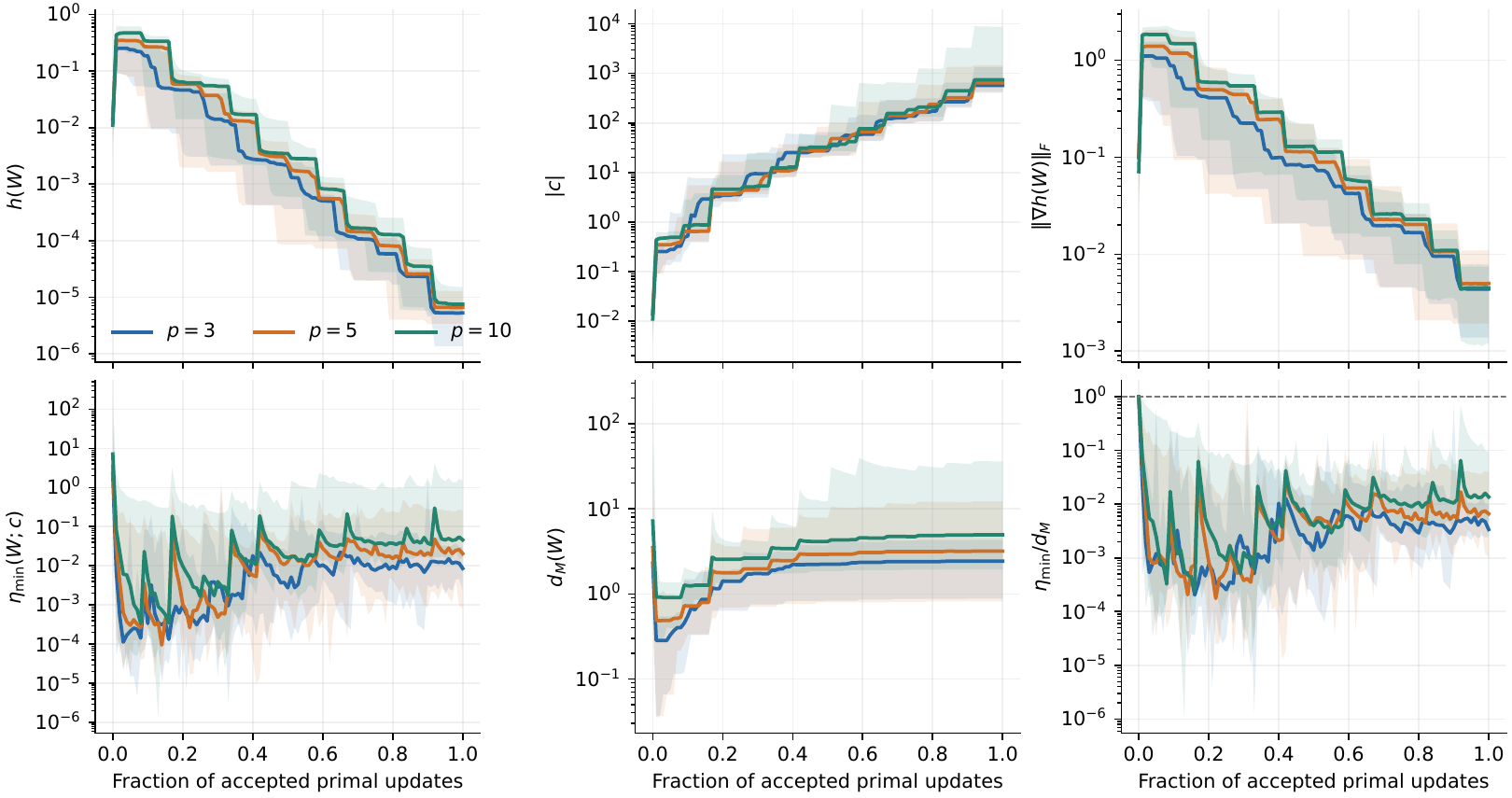}
\caption{\textbf{Stationarity along saved accepted primal outputs,}
using each update's pre-controller controls. At each of $101$ normalized
progress fractions $f$, trajectory $j$ contributes its saved index
$\lfloor f(n_j-1)\rfloor$. Lines are medians and bands are minima/maxima
over ten input-seed trajectories per dimension, not confidence
intervals. The horizontal axes do not align outer-stage boundaries.
The plotted first point is the first primal output; at the separately
recorded initial $W=0$, $h=\nabla h=c=0$ and $\eta_{\min}=d_M$.
Finite control changes can increase the augmented residual as the
constraint and its gradient decrease.}
\label{fig:ex3_stationarity}
\end{figure}


\begin{thebibliography}{}

\bibitem[Ahn et~al., 2023]{Ahn2023TransformersLT}
Ahn, K., Cheng, X., Daneshmand, H., and Sra, S. (2023).
\newblock Transformers learn to implement preconditioned gradient descent for
  in-context learning.
\newblock {\em ArXiv}, abs/2306.00297.

\bibitem[Aky{\"u}rek et~al., 2022]{Akyrek2022WhatLA}
Aky{\"u}rek, E., Schuurmans, D., Andreas, J., Ma, T.-Y., and Zhou, D. (2022).
\newblock What learning algorithm is in-context learning? investigations with
  linear models.
\newblock {\em ArXiv}, abs/2211.15661.

\bibitem[Aragam et~al., 2019]{Aragam2019GloballyOS}
Aragam, B., Amini, A.~A., and Zhou, Q. (2019).
\newblock Globally optimal score-based learning of directed acyclic graphs in
  high-dimensions.
\newblock In {\em Neural Information Processing Systems}.

\bibitem[Bai et~al., 2023]{Bai2023TransformersAS}
Bai, Y., Chen, F., Wang, H., Xiong, C.-M., and Mei, S. (2023).
\newblock Transformers as statisticians: Provable in-context learning with
  in-context algorithm selection.
\newblock {\em ArXiv}, abs/2306.04637.

\bibitem[Beck and Teboulle, 2009]{Beck2009AFI}
Beck, A. and Teboulle, M. (2009).
\newblock A fast iterative shrinkage-thresholding algorithm for linear inverse
  problems.
\newblock {\em SIAM J. Imaging Sci.}, 2:183--202.

\bibitem[Beinlich et~al., 1989]{Beinlich1989TheAM}
Beinlich, I.~A., Suermondt, H.~J., Chavez, R.~M., and Cooper, G.~F. (1989).
\newblock The alarm monitoring system: A case study with two probabilistic
  inference techniques for belief networks.
\newblock In {\em Conference on Artificial Intelligence in Medicine in Europe}.

\bibitem[Bello et~al., 2022]{Bello2022DAGMALD}
Bello, K., Aragam, B., and Ravikumar, P.~K. (2022).
\newblock Dagma: Learning dags via m-matrices and a log-determinant acyclicity
  characterization.
\newblock {\em ArXiv}, abs/2209.08037.

\bibitem[Chen et~al., 2023]{Chen2023TransformersLT}
Chen, T., Fu, D., Jia, R., and Sharan, V. (2023).
\newblock Transformers learn to achieve second-order convergence rates for
  in-context linear regression.
\newblock {\em Advances in Neural Information Processing Systems 37}.

\bibitem[Chen et~al., 2018]{Chen2018OnCD}
Chen, W., Drton, M., and Wang, Y.~S. (2018).
\newblock On causal discovery with an equal-variance assumption.
\newblock {\em Biometrika}.

\bibitem[Chickering, 2002]{Chickering2002OptimalSI}
Chickering, D.~M. (2002).
\newblock Optimal structure identification with greedy search.
\newblock {\em J. Mach. Learn. Res.}, 3:507--554.

\bibitem[Colombo and Maathuis, 2012]{Colombo2012OrderindependentCC}
Colombo, D. and Maathuis, M.~H. (2012).
\newblock Order-independent constraint-based causal structure learning.
\newblock {\em J. Mach. Learn. Res.}, 15:3741--3782.

\bibitem[Daneshmand, 2026]{daneshmand2026incontext}
Daneshmand, H. (2026).
\newblock In-context learning for discrete optimal transport: Can transformers
  sort?
\newblock In {\em The 29th International Conference on Artificial Intelligence
  and Statistics}.

\bibitem[Daubechies et~al., 2003]{Daubechies2003AnIT}
Daubechies, I., Defrise, M., and Mol, C.~D. (2003).
\newblock An iterative thresholding algorithm for linear inverse problems with
  a sparsity constraint.
\newblock {\em Communications on Pure and Applied Mathematics}, 57.

\bibitem[de~Luca and Fountoulakis, 2024]{deLuca2024SimulationOG}
de~Luca, A.~B. and Fountoulakis, K. (2024).
\newblock Simulation of graph algorithms with looped transformers.
\newblock {\em ArXiv}, abs/2402.01107.

\bibitem[Deng et~al., 2023]{Deng2023OptimizingNO}
Deng, C., Bello, K., Aragam, B., and Ravikumar, P.~K. (2023).
\newblock Optimizing notears objectives via topological swaps.
\newblock In {\em International Conference on Machine Learning}.

\bibitem[Garg et~al., 2022]{Garg2022WhatCT}
Garg, S., Tsipras, D., Liang, P., and Valiant, G. (2022).
\newblock What can transformers learn in-context? a case study of simple
  function classes.
\newblock {\em ArXiv}, abs/2208.01066.

\bibitem[Giannou et~al., 2023]{Giannou2023LoopedTA}
Giannou, A., Rajput, S., yong Sohn, J., Lee, K., Lee, J.~D., and
  Papailiopoulos, D. (2023).
\newblock Looped transformers as programmable computers.
\newblock {\em ArXiv}, abs/2301.13196.

\bibitem[Gregor and LeCun, 2010]{Gregor2010LearningFA}
Gregor, K. and LeCun, Y. (2010).
\newblock Learning fast approximations of sparse coding.
\newblock In {\em International Conference on Machine Learning}.

\bibitem[Jin et~al., 2024]{Jin2024RevisitingDS}
Jin, K., Ng, I., Zhang, K., and Huang, B. (2024).
\newblock Revisiting differentiable structure learning: Inconsistency of ℓ1
  penalty and beyond.
\newblock {\em ArXiv}, abs/2410.18396.

\bibitem[Kaiser and Sipos, 2021]{Kaiser2021UnsuitabilityON}
Kaiser, M. and Sipos, M. (2021).
\newblock Unsuitability of notears for causal graph discovery when dealing with
  dimensional quantities.
\newblock {\em Neural Processing Letters}, 54:1587 -- 1595.

\bibitem[Ke et~al., 2022]{Ke2022LearningTI}
Ke, N.~R., Chiappa, S., Wang, J.~X., Bornschein, J., Goyal, A., Rey, M., Weber,
  T., Botvinick, M.~M., Mozer, M.~C., and Rezende, D.~J. (2022).
\newblock Learning to induce causal structure.
\newblock {\em ArXiv}, abs/2204.04875.

\bibitem[Loh and B{\"u}hlmann, 2013]{Loh2013HighdimensionalLO}
Loh, P.-L. and B{\"u}hlmann, P. (2013).
\newblock High-dimensional learning of linear causal networks via inverse
  covariance estimation.
\newblock {\em J. Mach. Learn. Res.}, 15:3065--3105.

\bibitem[Lorch et~al., 2022]{Lorch2022AmortizedIF}
Lorch, L., Sussex, S., Rothfuss, J., Krause, A., and Sch{\"o}lkopf, B. (2022).
\newblock Amortized inference for causal structure learning.
\newblock {\em ArXiv}, abs/2205.12934.

\bibitem[Monga et~al., 2019]{Monga2019AlgorithmUI}
Monga, V., Li, Y., and Eldar, Y.~C. (2019).
\newblock Algorithm unrolling: Interpretable, efficient deep learning for
  signal and image processing.
\newblock {\em IEEE Signal Processing Magazine}, 38:18--44.

\bibitem[Nazaret et~al., 2023]{Nazaret2023StableDC}
Nazaret, A., Hong, J., Azizi, E., and Blei, D.~M. (2023).
\newblock Stable differentiable causal discovery.
\newblock {\em ArXiv}, abs/2311.10263.

\bibitem[Ng et~al., 2020a]{Ng2020OnTR}
Ng, I., Ghassami, A., and Zhang, K. (2020a).
\newblock On the role of sparsity and dag constraints for learning linear dags.
\newblock {\em ArXiv}, abs/2006.10201.

\bibitem[Ng et~al., 2023]{Ng2023StructureLW}
Ng, I., Huang, B., and Zhang, K. (2023).
\newblock Structure learning with continuous optimization: A sober look and
  beyond.
\newblock In {\em CLEaR}.

\bibitem[Ng et~al., 2020b]{Ng2020OnTC}
Ng, I., Lachapelle, S., Ke, N.~R., Lacoste-Julien, S., and Zhang, K. (2020b).
\newblock On the convergence of continuous constrained optimization for
  structure learning.
\newblock In {\em International Conference on Artificial Intelligence and
  Statistics}.

\bibitem[Nichani et~al., 2024]{Nichani2024HowTL}
Nichani, E., Damian, A., and Lee, J.~D. (2024).
\newblock How transformers learn causal structure with gradient descent.
\newblock {\em ArXiv}, abs/2402.14735.

\bibitem[Peters and Buhlmann, 2012]{Peters2012IdentifiabilityOG}
Peters, J. and Buhlmann, P. (2012).
\newblock Identifiability of gaussian structural equation models with equal
  error variances.
\newblock {\em Biometrika}, 101:219--228.

\bibitem[Peters and B{\"u}hlmann, 2015]{Peters2015StructuralID}
Peters, J. and B{\"u}hlmann, P. (2015).
\newblock Structural intervention distance for evaluating causal graphs.
\newblock {\em Neural Computation}, 27:771--799.

\bibitem[Reisach et~al., 2021]{Reisach2021BewareOT}
Reisach, A.~G., Seiler, C., and Weichwald, S. (2021).
\newblock Beware of the simulated dag! causal discovery benchmarks may be easy
  to game.
\newblock In {\em Neural Information Processing Systems}.

\bibitem[Reisach et~al., 2023]{Reisach2023ASS}
Reisach, A.~G., Tami, M., Seiler, C., Chambaz, A., and Weichwald, S. (2023).
\newblock A scale-invariant sorting criterion to find a causal order in
  additive noise models.
\newblock {\em Advances in Neural Information Processing Systems 36}.

\bibitem[Sachs et~al., 2005]{Sachs2005CausalPN}
Sachs, K., Perez, O.~D., Pe'er, D., Lauffenburger, D.~A., and Nolan, G.~P.
  (2005).
\newblock Causal protein-signaling networks derived from multiparameter
  single-cell data.
\newblock {\em Science}, 308:523 -- 529.

\bibitem[Seng et~al., 2024]{Seng2024LearningLD}
Seng, J., Zecevic, M., Dhami, D.~S., and Kersting, K. (2024).
\newblock Learning large dags is harder than you think: Many losses are minimal
  for the wrong dag.
\newblock In {\em International Conference on Learning Representations}.

\bibitem[Shazeer, 2020]{Shazeer2020GLUVI}
Shazeer, N. (2020).
\newblock Glu variants improve transformer.
\newblock {\em ArXiv}, abs/2002.05202.

\bibitem[Shimizu et~al., 2006]{Shimizu2006ALN}
Shimizu, S., Hoyer, P.~O., Hyv{\"a}rinen, A., and Kerminen, A.~J. (2006).
\newblock A linear non-gaussian acyclic model for causal discovery.
\newblock {\em J. Mach. Learn. Res.}, 7:2003--2030.

\bibitem[Spiegelhalter and Cowell, 1992]{Spiegelhalter1992Learning}
Spiegelhalter, D.~J. and Cowell, R.~G. (1992).
\newblock Learning in probabilistic expert systems.
\newblock In Bernardo, J.~M., Berger, J.~O., Dawid, P., Smith, A. F.~M.,
  Bernardo, J.~M., Berger, J.~O., Dawid, P., and Smith, A. F.~M., editors, {\em
  Bayesian Statistics 4: Proceedings of the Fourth Valencia International
  Meeting, Dedicated to the memory of Morris H. DeGroot, 1931–1989}. Oxford
  University Press.

\bibitem[van~de Geer and B{\"u}hlmann, 2012]{vandeGeer2012ell0penalizedML}
van~de Geer, S.~A. and B{\"u}hlmann, P. (2012).
\newblock $\ell_0$-penalized maximum likelihood for sparse directed acyclic
  graphs.
\newblock {\em arXiv: Statistics Theory}.

\bibitem[von Oswald et~al., 2022]{vonOswald2022TransformersLI}
von Oswald, J., Niklasson, E., Randazzo, E., Sacramento, J., Mordvintsev, A.,
  Zhmoginov, A., and Vladymyrov, M. (2022).
\newblock Transformers learn in-context by gradient descent.
\newblock In {\em International Conference on Machine Learning}.

\bibitem[Wang et~al., 2026]{Wang2026TransformerII}
Wang, X.-Y., Wang, S., and Huang, B.-W. (2026).
\newblock Transformer is inherently a causal learner.
\newblock {\em ArXiv}, abs/2601.05647.

\bibitem[Wei et~al., 2020]{Wei2020DAGsWN}
Wei, D., Gao, T., and Yu, Y. (2020).
\newblock Dags with no fears: A closer look at continuous optimization for
  learning bayesian networks.
\newblock {\em ArXiv}, abs/2010.09133.

\bibitem[Yu et~al., 2019]{Yu2019DAGGNNDS}
Yu, Y., Chen, J., Gao, T., and Yu, M. (2019).
\newblock Dag-gnn: Dag structure learning with graph neural networks.
\newblock {\em ArXiv}, abs/1904.10098.

\bibitem[Zhang et~al., 2022]{Zhang2022TruncatedMP}
Zhang, Z., Ng, I., Gong, D., Liu, Y., Abbasnejad, E., Gong, M.-M., Zhang, K.,
  and Shi, J.~Q. (2022).
\newblock Truncated matrix power iteration for differentiable dag learning.
\newblock {\em ArXiv}, abs/2208.14571.

\bibitem[Zheng et~al., 2018]{Zheng2018DAGsWN}
Zheng, X., Aragam, B., Ravikumar, P.~K., and Xing, E.~P. (2018).
\newblock Dags with no tears: Continuous optimization for structure learning.
\newblock In {\em Neural Information Processing Systems}.

\end{thebibliography}
\end{document}